\documentclass[preprint,12pt,3p]{elsarticle}

\usepackage{amssymb}
\usepackage{amsthm}
\usepackage{color}
\usepackage{subcaption}
\usepackage{amsmath,amssymb}
\usepackage{algorithmic}

\newtheorem{theorem}{Theorem}

\newtheorem{definition}{Definition}

\begin{document}

\begin{frontmatter}

\title{Addressing Expensive Multi-objective Games with Postponed Preference Articulation via Memetic Co-evolution}

\author[wutaddress]{Adam {\.Z}ychowski\corref{cor1}}\ead{a.zychowski@mini.pw.edu.pl}
\author[ntuaddress]{Abhishek Gupta}
\author[wutaddress]{Jacek Ma{\'n}dziuk}
\author[ntuaddress]{Yew Soon Ong}

\address[wutaddress]{Faculty of Mathematics and Information Science, Warsaw University of Technology, \\Koszykowa 75, 00-662 Warsaw, Poland}
\address[ntuaddress]{School of Computer Science and Engineering, Nanyang Technological University, \\Block N4, Nanyang Avenue, Singapore 639798}
\cortext[cor1]{Corresponding author}

\begin{abstract}
This paper presents algorithmic and empirical contributions demonstrating that the convergence characteristics of a co-evolutionary approach to tackle Multi-Objective Games (MOGs) with postponed preference articulation can often be hampered due to the possible emergence of the so-called Red Queen effect. Accordingly, it is hypothesized that the convergence characteristics can be significantly improved through the incorporation of memetics (local solution refinements as a form of lifelong learning), as a promising means of mitigating (or at least suppressing) the Red Queen phenomenon by providing a guiding hand to the purely genetic mechanisms of co-evolution. Our practical motivation is to address MOGs of a time-sensitive nature that are characterized by computationally expensive evaluations, wherein there is a natural need to reduce the total number of  true function evaluations consumed in achieving good quality solutions. To this end, we propose novel enhancements to co-evolutionary approaches for tackling MOGs, such that memetic local refinements can be efficiently applied on evolved candidate strategies by searching on computationally cheap surrogate payoff landscapes (that preserve postponed preference conditions). The efficacy of the proposal is demonstrated on a suite of test MOGs that have been designed.
\end{abstract}


\end{frontmatter}

% \linenumbers

%--------------------------------------------------------------

\section{Introduction}
\label{sec:Intrduction}

Many practical problems can be modeled and resolved with game theory methods. In real world problems decisions are usually made with multiple objectives or lists of payoffs. The notion of vector payoffs for games was originally introduced by Blackwell~\cite{Blackwell1956} and later extended by Contini~\cite{Contini1966}. Such games are named Multi-Objective Games (MOGs). MOGs may have many practical applications in engineering, economics, cybersecurity~\cite{Eisenstadt2017} or Security Games~\cite{Brown2014}, where real-life situations can be easily modeled as a game, and their solutions help decision-makers make the right choice in multi-objective environments.

Most existing studies about MOGs concentrate on differential games and defining the equilibrium for them. The most common approach is the Pareto-Nash equilibrium proposed in~\cite{Lozovanu2005}. The Pareto-Nash equilibrium uses the concept of cooperative games, because sub-players under the same coalitions should, according to the Pareto notion, optimize their vector functions on a set of strategies. This notion also takes into account the concept of non-cooperative games, because coalitions are interested in preserving the Nash equilibrium between coalitions. First attempts to solving MOGs were using multi-parametric criteria linear programming, for example in~\cite{Zeleny1975} (for zero-sum games) and~\cite{Contini1966} (for non-zero-sum variants). Also artificial intelligence-based approaches, such as fuzzy functions, have been applied to MOGs. For example, in~\cite{Sakawa2002}, the objectives are aggregated to one artificial objective for which the weights (players' preferences towards the objectives) are modeled with fuzzy functions.

Generally speaking, the most popular method of solving MOGs is to specify players objective preferences and define the utility function, for example a weighted sum, to transform the MOG into a surrogate Single Objective Game (SOG)~\cite{Shapley1959,Zeleny1975}. However, in real-life applications such an approach may not be sufficient, because preferences are often postponed until tradeoffs are revealed. Furthermore, in many cases decision-makers objectives are conflicting which makes specifying the utility function a priori more difficult.

There is a lack of literature on the topic of MOGs, especially which involves players with postponed preference articulation. Thus, there is a significant gap in the availability of algorithms for tackling real-world MOGs. First of all, an efficient numerical scheme is needed that is able to present decision-makers with the optimal tradeoffs in competitive game settings comprising multiple conflicting objectives (somewhat similarly to the case of standard multi-objective optimization~\cite{Liu2016,Mirjalili2017,Gholamian2007}). Only then can an informed postponed choice be made with regard to ascertaining the most preferred strategy to implement. The present paper takes a step towards filling this algorithmic void. First formalization of such MOGs can be found in~\cite{Eisenstadt2014}. In~\cite{Eisenstadt2016}, the definition of \textit{rationalizable} strategies in such games is provided together with a suggestion about how these strategies could be found (preliminary discussions are provided in Section~\ref{sec:preliminaries} of this paper), analyzed, and used for choosing the preferred strategy.

Co-evolutionary adaptation is a viable method for solving game theory problems and is successfully used in traditional SOGs~\cite{Yuhui2002,Fabris2012}. Preliminary works in~\cite{Eisenstadt2015} showed that in principle a co-evolutionary algorithm may be applied to MOGs as well, with the population-based evolutionary algorithms being particularly well-suited for handling multiple objectives simultaneously (as a consequence of the implicit parallelism of population-based search). However, a canonical co-evolutionary approach to solving MOGs has certain drawbacks. One of them is the emergence of a phenomenon named the Red Queen effect which often hampers the convergence characteristics of the algorithm. In many time-sensitive applications involving computationally expensive evaluations, such a slowdown must be avoided.

The Red Queen principle was first proposed by the evolutionary biologist L. van Valen in 1973~\cite{Valen1973}. It states that populations must continuously adapt to survive against ever-evolving competitors. It is based on a biologically grounded hypothesis that species continually need to change to keep up with the competitors (because when species stop changing, they will lose to the other species that do continue to change). The Red Queen effect could have positive as well as negative consequences. For example, co-evolution between predators and preys where the only way the predator can compensate for a better defense by the prey (e.g. rabbits running faster) is by developing a better offense (e.g. foxes running faster). This leads to improvement of the skills (running faster) of both species. In another example, consider trees in a forest which compete for access to sunlight. If one tree grows taller than its neighbours it can capture part of their sunlight. This causes the other trees to grow taller, in order not to be overshadowed. The effect is that all trees tend to become taller and taller, but still getting on average the same amount of sunlight. Optimizing access to sunlight for each individual tree does not lead to optimal performance for the forest as a whole~\cite{Heylighen1993}. Notably, such continuous adaptation as a result of the Red Queen effect occurs not only in species co-evolution, but also in disease mutations, business competitors or macroeconomics changes.

In a co-evolutionary algorithm for solving MOGs, the emergence of the Red Queen effect may hamper the convergence characteristics of the algorithm, since many function evaluations are needed to overcome the continuous adaptations and drive the population to a more-or-less steady state of reasonably good solutions. It  must be observed, however, that the presence of the Red Queen effect in a given co-evolutionary approach is - in general - only hypothetical as tracking the Red Queen phenomenon is usually a complex and non obvious process~\cite{Cliff1995}.
Regardless of the detailed reasons, the decelerated convergence is often only an artefact of the co-evolutionary method, and may not be at all related to the underlying MOGs. For the cases where MOGs are time-sensitive and/or involve computationally expensive evaluations, the Red Queen effect is unaffordable from an algorithmic standpoint. Thus, in this paper, an approach to mitigating (or at least suppressing) the deleterious consequence of the Red Queen effect is proposed. It is achieved by applying a local solution refinement technique (alternatively known as lifelong learning of an individual in a population) - in the spirit of memetic algorithms.

Canonical memetic algorithms~\cite{Ongetal2010,Ongetal2011,Neri2012} enhance population-based Evolutionary Algorithms (EA) by means of adding a distinctive local optimization phase. The underlying idea of memetics is to use local optimization techniques or domain knowledge to improve potential solutions (represented by individuals in a population) between consecutive EA generations. Drawing from their sociological interpretation, memes are seen as basic units of knowledge that are passed down from parents to offspring in the form of ideas (procedures) that serve as guidelines for superior performance. Thus, while genes leap from body to body as they propagate in the gene pool, memes are thought of as leaping from brain to brain as they propagate in the meme pool. A synergetic combination of simple local improvement schemes with evolutionary operators leads to complex and powerful solving methods which are applicable to a wide range of problems~\cite{Memetic:Review}, and currently serve as one of the fastest growing subfields of Evolutionary Computation research.
Local improvement of temporary solutions represented by genes is deemed to be of paramount importance in the context of the presumed existence of the Red Queen effect as it should strongly mitigate the rolling-horizon of the individual fitness evaluation. The main rationale for such a claim is that the lifelong learning module introduced by memetics can provide a guiding hand to the purely genetic mechanisms of co-evolutionary algorithms, thereby potentially suppressing the intensity of the Red Queen effect in our search for equilibrium solutions to the underlying MOG. Due to the complexity of rigorous theoretical analysis, we attempt to substantiate our claims experimentally in this paper.

To summarize, the main contribution of this paper is a novel enhancement of co-evolutionary algorithms for MOGs. In particular, memetic local refinements are proposed on evolved candidate strategies, as a means of improving convergence behavior. Importantly, in order to make such local refinements computationally viable in competitive multi-objective game settings, we incorporate an approach that reduces sets of payoff vectors in objective space to a single representative point that preserves the postponed preference articulation condition (details are provided in Section~\ref{sec:canonicalAlgorithm}). Thereafter, a surrogate model of the representative point can be built, which allows the local improvements to be carried out efficiently by searching on the surrogate landscape~\cite{Goh2011,Lim2010}. As a result of the proposal, it is considered that MOGs of a computationally expensive nature can be effectively handled, at negligible computational overhead involved in surrogate modeling.

The remainder of this paper is arranged as follows. Section~\ref{sec:preliminaries} presents the general formulation and fundamentals of MOGs. An overview of the baseline co-evolutionary approach for solving MOGs based on~\cite{Eisenstadt2015} is presented in Section~\ref{sec:canonicalAlgorithm}. Section~\ref{sec:memeticAlgorithm} provides a more detailed description of the proposed surrogate-assisted memetic algorithm tailored for computationally expensive MOGs. In Section~\ref{sec:experiments} experimental studies for both algorithms are carried out on a suite of test MOGs of varying degrees of complexity that have been designed based on an intuitively visualizable differential game. Results are discussed in the context of convergence and suppression of the Red Queen effect. The last section is devoted to conclusions and directions for future research.

%--------------------------------------------------------------

\section{Preliminaries on MOGs}
\label{sec:preliminaries}

\subsection{Problem definition}

Single-act multi-objective games considered in this paper can be formally described as follows. Let $P_1$ and $P_2$ be the players competing against each other, and $S_1 = \{\boldsymbol{s^1_1}, \boldsymbol{s^2_1}, \ldots , \boldsymbol{s^I_1}\}$, $S_2 = \{\boldsymbol{s^1_2}, \boldsymbol{s^2_2}, \ldots , \boldsymbol{s^J_2}\}$ be complete sets of their possible strategies, respectively. Each candidate strategy is a vector of decision parameters: $\boldsymbol{s^i_1} = [s_1^{i(1)}, s_1^{i(2)}, \ldots, s_1^{i(N_1)}] \in \mathbb{R}^{N_1}, \boldsymbol{s^j_2} = [s_2^{j(1)}, s_2^{j(2)}, \ldots, s_2^{j(N_2)}] \in \mathbb{R}^{N_2}$, where $N_1$ and $N_2$ are the numbers of players' decision parameters. $\boldsymbol{\bar{f}_{i,j}}$ represents the payoff vector corresponding to a game evaluated using strategies $\boldsymbol{s_1^i}$ and $\boldsymbol{s_2^j}$. Without loss of generality, we assume that the goal of $P_1$ is to minimize the payoff, while the goal of $P_2$ is to maximize the payoff.

Since both players do not know how the opponent evaluates their objectives in a postponed preference articulation setting, each player takes a worst-case approach. Thus, players must somehow take into account all possible moves available to the opponent. From  $P_1$'s perspective, the goal may be modeled as minimizing the
objective function vector assuming the best possible opponent strategies: $min_{s_1^i\in S_1}max_{s_2^j\in S_2}\boldsymbol{\bar{f}_{i,j}}$. In contrast, player $P_2$ aims at maximizing the objectives while considering the best strategies for the minimizer: $max_{s_2^j\in S_2}min_{s_1^i\in S_1}\boldsymbol{\bar{f}_{i,j}}$. Note that, the above \textit{max} and \textit{min} operators applied to vector-valued payoffs are ill-defined. As one possible alternative, their meaning can be formalized by means of \textit{domination relations} between payoff vectors, as described in subsection~\ref{sec:dominationrelations} below.

\subsection{Solution approach}

Contrary to SOGs, in games with multiple objectives, a universally optimal strategy usually does not exist. For this reason, the notion of Pareto optimality is used based on domination relations between individual vectors as well as sets of vectors. Most importantly, such deductions can be found without the need to specify objective preferences, which aligns well with our basic premise of MOGs with postponed preference articulation.

Definitions presented in the reminder of this section follow the discussions in~\cite{Eisenstadt2016}. Herein, we only provide a brief overview of the main ideas for the sake of brevity.

\subsection{Domination relations}
\label{sec:dominationrelations}

To resolve domination relation between sets of vectors, first the domination relation between individual vectors must be defined.

\begin{definition}{\textbf{Domination relation between vectors}}

Let $\textbf{f} = [f_1,f_2,\ldots,f_K] \in \mathbb{R}^K $ and $\textbf{h} = [h_1,h_2,\ldots,h_K] \in \mathbb{R}^K$ be two vectors in the objective space. A vector $\textbf{f}$ \textit{dominates} vector $\textbf{h}$ in a maximization problem $(\textbf{f} \overset{\text{max}}\succ \textbf{h})$, if $f_k \geq h_k$ for all $k\in \{1,\ldots,K\}$ and there exists $k\in \{1,\ldots,K\}$ for which $f_k > h_k$.
\end{definition}

With this, the domination relation between sets is defined as follows.

\begin{definition}{\textbf{Domination relation between sets}}

Let $F$ and $H$ be sets of vectors from the objective space.
Set $F$ \textit{dominates} set $H$ in a maximization problem, $F \overset{\text{max}}\succ H$, if $\forall h \in H, \exists f \in F$, such that $f \overset{\text{max}}\succ h$.
\end{definition}

Analogous definitions to the above are used for domination relations $(\overset{\text{min}}\succ)$ in minimization problems.

In the context of  MOGs, the notion of 'worst case domination' emerges in addition to the simple domination relation between sets, because, while assessing the payoff of a particular strategy for $P_1$ (or $P_2$) in a competing game, the set of optimal strategies for the opponent must be taken into account. Accordingly, we define the \textit{worst-case domination} relation. 

\begin{definition}{\textbf{Worst-case domination}}

Set $F$ \textit{worst-case dominates} set $H$ ($F \overset{\text{w.c.}}\succ H$) in a maximization problem when $H \overset{\text{min}}\succ F$, and set $F$ worst-case dominates set $H$ in a minimization problem when $H \overset{\text{max}}\succ F$.
\end{definition}

\subsection{Rationalizable strategies}

Given the worst-case domination relation, Pareto-optimality in a MOG is defined as follows.

\begin{definition}{\textbf{Pareto-optimality in MOGs}}

A set $Z^*$ is \textit{Pareto-optimal} in a MOG if no other set exists that worst-case dominates $Z^*$. The set of all worst case non-dominated sets constitutes the Pareto set of sets $P^*$ (also referred to as the \textit{Pareto Layer}~\cite{Avigad2010}):
$$P^*=\{F: \neg \exists H s.t. H \overset{\text{w.c.}}\succ F \}$$
\end{definition}

To elaborate from the point of view of the players in the MOG, when evaluating the $i$-th strategy of player $P_1$ ($\boldsymbol{s_1^i}$), there are $J$ possible strategies $[\boldsymbol{s_2^1}, \boldsymbol{s_2^2}, \ldots \boldsymbol{s_2^J}]$ of $P_2$ to consider. If the objective preferences of $P_2$ are not defined, then there is a set of non-dominated payoff vectors, which are in fact all possible best responses of $P_2$ to the strategy $\boldsymbol{s_1^i}$. This set is named the \textit{anti-optimal front} ($F_{s^i_1}^{-*}$) corresponding to the $i$-th strategy of $P_1$.

\begin{definition}{\textbf{Irrational strategies}}

A set of \textit{irrational strategies} ($S_1^{irr}$) of player $P_1$ is defined as follows:
$$S_1^{irr} = \{\boldsymbol{s_1^i} \in S_1\ |\ \exists \boldsymbol{s_1^{i'}} \in S_1\ F_{s^{i'}_1}^{-*} \overset{\text{w.c.}}\succ F_{s^i_1}^{-*}\ \forall i \in \{1,\ldots,I\}\}$$
\end{definition}

All strategies which are not irrational are called rationalizable.

\begin{definition}{\textbf{Rationalizable strategies}}

A set of \textit{rationalizable strategies} of player $P_1$ is defined as $S_1^R=S_1-S_1^{irr}$.
\end{definition}

A detailed discussion with examples concerning domination relations and preferable outcomes can be found in the Appendix of~\cite{Eisenstadt2016}. The main conclusion stemming from that discussion is that, for a particular player in a MOG, if the anti-optimal front corresponding to strategy $\boldsymbol{s^i}$ worst-case dominates the anti-optimal front corresponding to $\boldsymbol{s^j}$, then $\boldsymbol{s^i}$ always produces a preferable outcome for that player. On the other hand, when two anti-optimal fronts are worst-case non-dominated, one of the strategies could be a better choice for a certain objective preference articulation, while the other strategy may be better under some other preferences. Therefore, in the latter case, their direct comparison is not possible.

In this paper, the described worst-case domination approach is used as the basic tool to solve MOGs. All considered MOGs are assumed to have postponed preferences, and candidate strategies are accordingly analyzed from both players' perspectives in the proposed co-evolutionary algorithms (described next).

%--------------------------------------------------------------

\section{The Canonical Co-evolutionary Algorithm for MOGs}
\label{sec:canonicalAlgorithm}

Our implementation of the canonical co-evolutionary algorithm for MOGs (Canonical CoEvoMOG) is designed based on~\cite{Eisenstadt2015}. Each subpopulation in the algorithm caters to a unique player in the MOG. The key difference between our implementation and that of~\cite{Eisenstadt2015} is that, instead of considering the entire anti-optimal set of vectors while determining worst-case domination relations, we only consider the ideal point of the anti-optimal set as a single representative vector. Note that the term "ideal" is used from the point of view of the opponent. Thus, without loss of generality, if the anti-optimal front corresponding to strategy $\boldsymbol{s^i_1}$ of the minimization player $P_1$ is $F_{s^i_1}^{-*}$, then the ideal point is defined by the maximum (extreme) individual objective values that occur in $F_{s^i_1}^{-*}$. An illustration is depicted in Figure~\ref{fig:idealpoint}, where the set of white circles (representing the anti-optimal front corresponding to $\boldsymbol{s_1^i}$) worst-case dominate the set of white squares (representing the anti-optimal front corresponding to $\boldsymbol{s_1^j}$) in minimization sense. The maximizing opponent's ideal points, given $\boldsymbol{s_1^i}$ or $\boldsymbol{s_1^j}$, are also shown in the figure in black. From Figure~\ref{fig:idealpoint}, we observe that while ascertaining the expected payoff of a particular strategy, the ideal point of the anti-optimal front can be seen as a meaningful representative encompassing all possible moves of the opponent. It is worth mentioning that as the ideal point is composed of the extreme values of all objectives, no specific objective preference is assumed for the opponent. In that sense, the proposed simplification can be seen as preserving the postponed preference articulation condition of the MOG.

\begin{figure}
	\begin{center}
    \includegraphics[width=0.35\columnwidth]{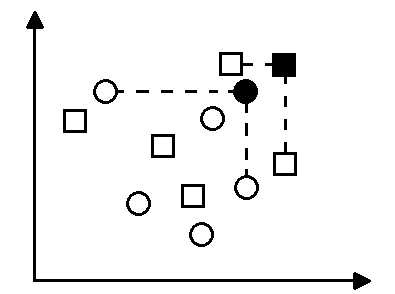}
    \caption{Example of payoff vectors in two-dimensional objective space. The white circles representing the anti-optimal front corresponding to $\boldsymbol{s_1^i}$ worst-case dominate the white squares representing the anti-optimal front corresponding to $\boldsymbol{s_1^j}$ in minimization sense. Black figures represent the opponent's ideal points for the respective sets.}
    \label{fig:idealpoint}
	\end{center}
\end{figure}

We emphasize the rationale behind using only a single representative vector (instead of the entire anti-optimal set) following the observation presented in~\cite{Avigad2013}, which can also be stated through the theorem below.

\begin{theorem}
If strategy $\boldsymbol{s^i}$ worst-case dominates $\boldsymbol{s^j}$, then the ideal point of the anti-optimal front of the former either dominates or is at least equal to the ideal point of the anti-optimal front of the latter strategy.
\end{theorem}

\begin{proof}
It follows from the definition of worst-case domination presented in the previous section. For brevity, we consider the strategies $\boldsymbol{s^i}$ and $\boldsymbol{s^j}$ of the minimization player. Similar arguments can be applied to the maximization player as well. Thus, the antecedent statement of the theorem implies that the anti-optimal front of $\boldsymbol{s^i}$ is maximization dominated by the anti-optimal front of $\boldsymbol{s^j}$. Accordingly, there exist vector(s) in the anti-optimal set of $\boldsymbol{s^j}$ that maximization dominate the extreme vectors of the anti-optimal front of $\boldsymbol{s^i}$. As a result, if we assume that the opponent's ideal point, given $\boldsymbol{s^i}$, maximization dominates the opponent's ideal point given $\boldsymbol{s^j}$, then we have a contradiction. In other words, the opponent's ideal point, given $\boldsymbol{s^i}$, must minimization dominate or be equal to the opponent's ideal point given $\boldsymbol{s^j}$.
\end{proof}

From an algorithmic point of view, the first advantage of using the representative ideal point vector corresponding to a particular strategy is that it allows us to directly employ standard non-domination relations between vectors (as in Definition 1), bypassing the cumbersome process of comparing sets of vectors to determine worst-case domination relations. In other words, from minimization player $P_1$'s standpoint, strategy $\boldsymbol{s_1^i}$ is preferred over $\boldsymbol{s_1^j}$ simply if the ideal point of the anti-optimal front of $\boldsymbol{s_1^i}$ dominates that of $\boldsymbol{s_1^j}$ in the minimization sense. Furthermore, the reduction of a set of vectors to a single point implies that simple diversity measures (such as the crowding distance~\cite{Deb2001}) may also be directly incorporated to facilitate a good distribution of alternative payoff vectors in the objective space of the MOG. Based on these basic ingredients, Figure~\ref{alg:CanonicalCoEvoMOG_code} outlines the schematic workflow of the Canonical CoEvoMOG algorithm.

In the Canonical CoEvoMOG algorithm, there are two subpopulations catering to the two players in the MOG. The method proceeds as in standard co-evolution for SOGs, where interactions are considered between all candidate strategies in the two subpopulations (forming a complete bipartite evaluation framework). The outcomes of the interactions are used to "approximate" the ideal point of the anti-optimal front corresponding to every candidate strategy of both players. The approximated ideal point vectors are then used to calculate non-domination ranks and crowding distances of strategies within each subpopulation separately, similarly to the case of evolutionary multi-objective optimization~\cite{Deb2001}. The non-domination ranks and the crowding distances are considered lexicographically for selecting the most promising candidate strategies in each subpopulation that progress the search to the next generation through genetic operations of crossover and mutation.

\begin{figure}[t]
\begin{small}
  \begin{algorithmic}
    \STATE{Initialize randomly population $S_1$ for player $P_1$ and population $S_2$ for player $P_2$.}
    
    \FORALL{generations}
        \STATE{\textbf{Step 1}: Create offspring populations $S'_1$ and $S'_2$ via crossover and mutation of parent individuals from $S_1$ and $S_2$, respectively.}
        
        \STATE{\textbf{Step 2}: Set $S_1''$ as $S_1 \cup S_1'$, and set $S_2''$ as $S_2 \cup S_2'$.}

        \STATE{\textbf{Step 3}: Evaluate populations $S_1''$ and $S_2''$ by performing all interactions between candidate strategies in $S_1''$ and $S_2''$ (keep track of evaluations to prevent repetitions).}
        
        \STATE{\textbf{Step 4}: Obtain the "approximate" ideal point corresponding to each candidate strategy in $S_1''$ and $S_2''$ based on the outcomes of all possible interactions.}
        
        \STATE{\textbf{Step 5}: Obtain non-domination rank and crowding distance of each strategy in $S_1''$ and $S_2''$ based on the approximated ideal point vectors.}
        
        \STATE{\textbf{Step 6}: Consider the non-domination ranks and crowding distances lexicographically to select the most promising candidate strategies from $S_1''$ and $S_2''$ to form $S_1$ and $S_2$ in the next generation.}

    \ENDFOR
\end{algorithmic}
\end{small}
\caption{Pseudo-code of the Canonical Co-evolutionary MOG Algorithm.}
\label{alg:CanonicalCoEvoMOG_code}
\end{figure}

%--------------------------------------------------------------
\section{The Memetic Co-evolutionary Algorithm for MOGs}
\label{sec:memeticAlgorithm}

One of the drawbacks faced by the Canonical CoEvoMOG algorithm is the possible emergence of the Red Queen effect. This suggests that the convergence to the desired equilibrium strategies is impeded due to the continuously adapting subpopulations that endlessly try to keep pace with the changes in the opponent's strategies. Notably, the slowdown is unlikely to be related to the underlying MOG, but is often an artefact of the co-evolutionary method itself. It is regarded that in MOG applications of a time-sensitive nature that may even involve computationally expensive evaluations, such a slowdown is not affordable. Therefore, in this section, we propose memetic local strategy refinements as a way of enhancing the purely genetic mechanisms of the Canonical CoEvoMOG algorithm, thereby potentially speeding up the convergence characteristics. Further, in order to maintain the computational feasibility of the method, a surrogate modeling of the representative payoff vector is proposed, which allows the local refinements to be carried out efficiently on the surrogate landscape. Our proposal is labeled as a Memetic CoEvoMOG algorithm, and involves a simple but important modification to the pseudo-code in Figure~\ref{alg:CanonicalCoEvoMOG_code}. Details of the modified procedure are presented in Figure~\ref{alg:MemeticCoEvoMOG_code}.

\begin{figure}[t]
\begin{small}
  \begin{algorithmic}
    \STATE{Randomly generate initial population $S_1$ for player $P_1$ and $S_2$ for $P_2$.}

    \STATE{Evaluate strategies in $S_1$ and $S_2$ considering all interactions possible.}

    \STATE{Train FNNs mapping candidate strategies to the corresponding ideal point approximations.}
    
    \FORALL{generations}
        \STATE{\textbf{Step 1}: Create offspring populations $S_1'$ and $S_2'$ via crossover and mutation of parent individuals from $S_1$ and $S_2$, respectively.}

        \STATE{\textbf{Step 2}: Apply local refinements using the surrogate landscape on a subset of randomly chosen individuals from $S_1'$ and $S_2'$ (see Figure~\ref{alg:MemeticCoEvoMOG_code2} for details).}
        
        \STATE{\textbf{Step 3}: Set $S_1''$ as $S_1 \cup S_1'$, and set $S_2''$ as $S_2 \cup S_2'$.}
        
        \STATE{\textbf{Step 4}: Evaluate populations $S_1''$ and $S_2''$ by performing all interactions between candidate strategies in $S_1''$ and $S_2''$ (keep track of evaluations to prevent repetitions).}
        
        \STATE{\textbf{Step 5}: Obtain the "approximate" ideal point corresponding to each candidate strategy in $S_1''$ and $S_2''$ based on the outcomes of all possible interactions.}
        
        \STATE{\textbf{Step 6}: Obtain non-domination rank and crowding distance of each strategy in $S_1''$ and $S_2''$ based on the approximated ideal point vectors.}
        
        \STATE{\textbf{Step 7}: Consider the non-domination ranks and crowding distances lexicographically to select the most promising candidate strategies from $S_1''$ and $S_2''$ to form $S_1$ and $S_2$ in the next generation.}
        
        \STATE{\textbf{Step 8}: Retrain FNNs based on $S_1$, $S_2$ and the corresponding ideal point approximations.}
    \ENDFOR
\end{algorithmic}
\end{small}
\caption{Pseudo-code of the Memetic Co-evolutionary MOG Algorithm.}
\label{alg:MemeticCoEvoMOG_code}
\end{figure}

\begin{figure}[t]
\begin{small}
  \begin{algorithmic}
    \STATE{Let probability of local search be $p_{ls}$.}
    
    \FORALL{individual in $S_1$}
        \STATE{\textbf{Step 1}: Select the individual with probability $p_{ls}$. If not selected, then continue to next individual.}

        \STATE{\textbf{Step 2}: Generate a random weight vector satisfying partition of unity.}
        
        \STATE{\textbf{Step 3}: Combine the FNN surrogates using the random weights for scalarization. Minimize the resultant objective via the Nelder-Mead simplex algorithm where the individual's strategy is taken as the starting point for local search.}
        
        \STATE{\textbf{Step 4}: Update the individual with the best solution found after the local search procedure.}
    \ENDFOR
\end{algorithmic}
\end{small}
\caption{Pseudo-code of Memetics via Local Refinement. Steps are shown herein from the perspective of the minimization player $P_1$. The procedure is trivially generalized to the case of the maximization player $P_2$ as well.}
\label{alg:MemeticCoEvoMOG_code2}
\end{figure}

\subsection{Overview of Surrogate Modeling in MOGs}

A surrogate model is essentially a computationally cheap approximation of the underlying (expensive) function to be evaluated. By searching on the surrogate landscape instead of the original function, significant savings in computational effort can be achieved~\cite{Jin2011,Loshchilov2013}. However, before building the surrogate model, the function(s) to be approximated must first be ascertained. For a MOG, this task is in general unclear, as corresponding to a particular strategy, a set of optimal opponent strategies usually exist that constitute the anti-optimal front. 

It is at this juncture that the second, and perhaps most relevant, implication of using the representative ideal point vector (instead of the entire anti-optimal front) is revealed. Without the proposed modification, it is difficult to imagine an approach for incorporating memetics into the canonical co-evolutionary algorithm for MOGs. To elaborate, while creating a surrogate of an entire set of vectors is indeed prohibitive, surrogate models that map an individual strategy to its corresponding ideal point vector (of the anti-optimal set) can presumably be learned with relative ease. 

In the Memetic CoEvoMOG algorithm, the data generated for $S_1$ and $S_2$ at Step 6 of Figure~\ref{alg:CanonicalCoEvoMOG_code} is used for iterative surrogate modeling. Candidate strategies in $S_1$ and $S_2$ serve as the inputs to the surrogate model, while the corresponding approximate ideal point objectives serve as outputs of interest. Note that separate surrogate models are learned for each  player. Further, the models are learned repeatedly at every generation of the Memetic CoEvoMOG algorithm based on only the data generated during that generation (i.e., data is not accumulated across generations). The rationale behind this step is that the approximated ideal point vector tends to continuously adapt in a co-evolutionary algorithm in conjunction with the evolving strategies of the opponent, such that there may be little correlation in the data across generations. Finally, it must be mentioned that a simple feedforward neural network (FNN) is used for surrogate modeling in this paper, although any other preferred model type may also be incorporated with minimal change to the overall algorithmic framework.

\subsection{Memetics via Local Refinement}

Memetics in stochastic optimization algorithms (such as EAs) are generally realized via a local solution refinement step as a form of lifelong learning of individuals. Since the original functions are assumed to be computationally expensive, herein, the local refinements are carried out in the surrogate landscape. Notably, since the ideal point is itself vector-valued, the local search is performed by first reducing the vector to a scalar value via a simple random weighting of objectives; we ensure that the randomly generated weights satisfy the partition of unity condition. It is important to mention here that the random weights are sampled from a uniform probability density function, such that no biased preference information is imposed (which preserves the postponed preference articulation condition of the MOG). The local search method used in this study is the popular derivative-free (bounded) Nelder-Mead simplex algorithm, although any other algorithm may also be used. Thus, for the minimization player $P_1$, the simplex algorithm locally minimizes the randomly scalarized objectives, while for the maximization player $P_2$, the simplex algorithm locally maximizes the scalarized objectives. 

After offspring creation through genetic operations, a subset of them from both subpopulations of the Memetic CoEvoMOG algorithm are randomly selected with some user defined probability for local search. Once the local refinement is completed, i.e., the Nelder-Mead simplex algorithm converges to a point within the specified search space bounds, the improved solution (or strategy) is directly injected into the offspring population in the spirit of Lamarckian learning~\cite{Ong2004}. A brief overview of the steps involved in the memetic local refinement procedure is presented in Figure~\ref{alg:MemeticCoEvoMOG_code2}.

%--------------------------------------------------------------

\section{Numerical Experiments}
\label{sec:experiments}

The proposed method was tested on a simple differential MOG named \textit{tug-of-war}. The basic form of the game consists of a point with mass $m$ placed at coordinates $(0,0)$. Two players $P_1$ and $P_2$ choose angles $\theta_1$ and $\theta_2$, respectively, at which the respective forces with magnitudes $F_1,F_2$ are applied (see Figure~\ref{fig:tugofwar}). The game outcome is the position (given by coordinates $(x_1,x_2)$) of the mass $m$ after particular time $t_f$. The objectives of player $P_1$ are to minimize $x_1$ and $x_2$ and the objectives of player $P_2$ are to maximize these two coordinates.
The final coordinates can be computed with formulas:
$x_1 = (F_1 cos(\theta_1)+F_2 cos(\theta_2))\cdot\frac{1}{2}t^2_f$, $x_2 = (F_1 sin(\theta_1)+F_2 sin(\theta_2))\cdot\frac{1}{2}t^2_f$. For simplification, the following assumptions are made: $F_1=F_2=1$, $t_f=\sqrt{2}$, and thus $x_1=cos(\theta_1)+cos(\theta_2)$ and $x_2=sin(\theta_1)+sin(\theta_2)$.
In this game, the strategic decision is to choose the angles $\theta_1$ and $\theta_2$, so the space of possible strategies is infinite, since $\theta_1,\theta_2 \in [0;2\pi]$. Accordingly, observe that the continuous search space of $\theta_1$ corresponds to the set $S_1$, and that of $\theta_2$ corresponds to the set $S_2$.

\begin{figure}
	\begin{center}
    \includegraphics[width=0.4\columnwidth]{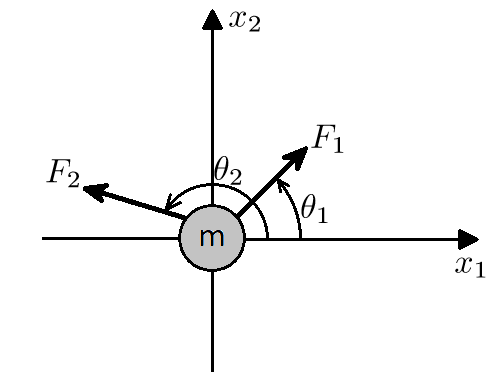}
    \caption{The tug-of-war game setting.}
    \label{fig:tugofwar}
	\end{center}
\end{figure}

Rationalizable strategies for both players can be intuitively ascertained. Player $P_1$ (minimizer) aims at having the mass in a position with negative coordinates (third quadrant), and therefore rationalizable strategies are in the range $\pi \leq \theta_1 \leq \frac{3}{2} \pi$. Similarly, player $P_2$ (maximizer) wants the mass to be located in the first quadrant, thus the rationalizable strategies are in the range $0 \leq \theta_2 \leq \frac{1}{2}\pi$. Due to postponed preference articulation, any move in the above ranges is a valid selection. Figure~\ref{fig:optimal_solution} shows all possible end game positions in the case of optimal performance, i.e. when both players select from the rationalizable range of strategies mentioned above.

\begin{figure}
	\begin{center}
    \includegraphics[width=0.5\columnwidth]{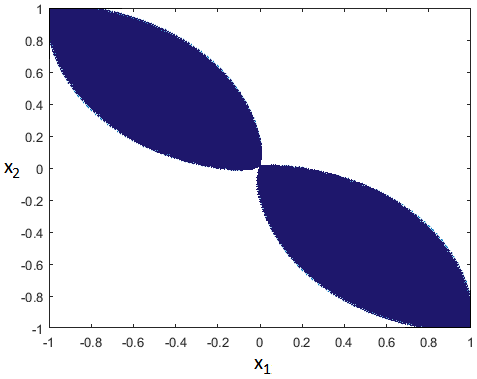}
    \caption{Representative set of all points in the objective space that reflect the rationalizable strategies of the tug-of-war MOG.}
    \label{fig:optimal_solution}
	\end{center}
\end{figure}

Due to the a priori known optimal performances of this MOG, the results of the algorithms applied to the tug-of-war game can be easily compared based on their approximation quality. Nevertheless, doing so is usually not possible in arbitrary MOGs where exact results are often too hard to compute in general practical settings. A more detailed description of the \textit{tug-of-war} game can be found in~\cite{Eisenstadt2015} where it was first used to test a version of the Canonical Co-evolutionary Algorithm similar to the one described in Section~\ref{sec:canonicalAlgorithm}.

\subsection{Experimental setup}

To make the tug-of-war MOG even more complex for the purpose of rigorous experimental testing, several synthetic functions have been artificially incorporated into the game formulation to create a number of alternate benchmarks. To elaborate, we define:
$$x_1 = \frac{F_1}{1+\phi(\boldsymbol{z_1})}cos(\theta_1)+\frac{F_2}{1+\phi(\boldsymbol{z_2})}cos(\theta_2), x_2 = \frac{F_1}{1+\phi(\boldsymbol{z_1})}sin(\theta_1)+\frac{F_2}{1+\phi(\boldsymbol{z_2})}sin(\theta_2),$$
where $\phi$ is the incorporated function, $z_1$ and $z_2$ are additional decision parameters introduced for $P_1$ and $P_2$, respectively, while $F_1,F_2 \in [0,1]$ are force magnitudes (which are now also treated as decision parameters).

In this way, several tug-of-war variants can be created. Tested functions were chosen to check the efficacy of the proposed methods under varying conditions. Note that the selected functions are widely used in the literature to evaluate global optimization methods, including evolutionary techniques. In particular, the following $9$ widely-known optimization functions were considered to serve as $\phi$: Rosenbrock 2D, Rosenbrock 3D, Rastrigin 1D, Rastrigin 2D, Rastrigin 3D, Griewank 1D, Griewank 2D, Griewank 3D and Ackley 2D. Their plots and detailed description of properties can be found in~\cite{Molga2013}. It is worth mentioning that as all the selected functions have minimum value $0$, the representation of rationalizable strategies of all MOG variants is the same as shown in Figure~\ref{fig:optimal_solution}.

\begin{enumerate}
    \item Rosenbrock $n$D
    \begin{small}
    $$\phi_1(\boldsymbol{z}) = \sum_{i=1}^{n-1}[ 100(\boldsymbol{z}(i+1)-\boldsymbol{z}(i)^2)^2+(\boldsymbol{z}(i)-1)^2] $$
    \end{small}
    Global minimum: $\phi_1(1,\ldots,1)=0$.

    \item Rastrigin $n$D
    \begin{small}
    $$\phi_2(\boldsymbol{z}) = 10n + \sum_{i=1}^n [\boldsymbol{z}(i)^2 - 10cos(2\pi \boldsymbol{z}(i))] $$
    \end{small}
    Global minimum: $\phi_2(0,\ldots,0)=0$.

    \item Griewank $n$D
    \begin{small}
    $$\phi_3(\boldsymbol{z}) = \sum_{i=1}^n \frac{\boldsymbol{z}(i)^2}{4000} - \prod_{i=1}^n cos(\frac{\boldsymbol{z}(i)}{\sqrt{i}}) + 1 $$
    \end{small}
    Global minimum: $\phi_3(0,\ldots,0)=0$.

    \item Ackley $n$D
    \begin{small}
    $$\phi_4(\boldsymbol{z}) = -20 \cdot exp[-0.2 \sqrt{\frac{1}{n}\sum_{i=1}^n\boldsymbol{z}(i)^2}] - exp[\frac{1}{n}\sum_{i=1}^n cos(2\pi \boldsymbol{z}(i))] + 20 + e$$
    \end{small}
    Global minimum: $\phi_4(0,\ldots,0)=0$.

\end{enumerate}

In the experimental study, the canonical as well as the memetic co-evolutionay algorithms were run with the same hyperparameter settings in order to ensure that any performance differences are indeed a consequence of the proposed memetics module. The size of each subpopulation in the co-evolutionary algorithms was taken as $50$, and the methods were run for $100$ generations. For recombination operations, simulated binary crossover (SBX)~\cite{Agrawal1995} was used with distribution index of $20$, and mutations were applied using the polynomial mutation operator~\cite{Deb1999} also with distribution index $20$. In the Memetic CoEvoMOG algorithm, the probability of local search on the surrogate landscape ($p_{ls}$) was set to $20\%$ throughout. The test problems were assumed to be computationally expensive, such that the extra computational effort spent on building and searching the surrogate model was considered negligible in comparison to the cost of evaluations of the true underlying problem. For many real-world settings, such an assumption on the cost of surrogate assistance is reasonable, and is commonly used in most surrogate-assisted optimization studies. For this reason, the comparison plots presented in the next subsection are drawn on the basis of the results achieved over certain number of generations in the co-evolutionary algorithms, rather than explicitly taking computational time into account.

\begin{figure}[ht!]
\captionsetup[subfigure]{labelformat=empty}

\begin{subfigure}{0.5\textwidth}
\includegraphics[width=0.99\linewidth, height=4cm]{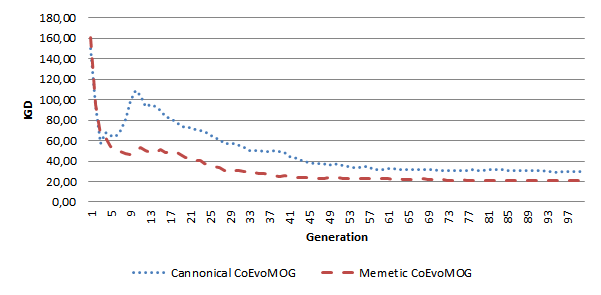}
\caption{Rosenbrock 2D}
\end{subfigure}
\begin{subfigure}{0.5\textwidth}
\includegraphics[width=0.99\linewidth, height=4cm]{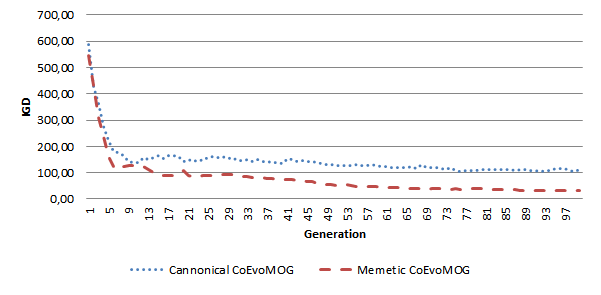}
\caption{Rosenbrock 3D}
\end{subfigure}

\begin{subfigure}{0.5\textwidth}
\includegraphics[width=0.99\linewidth, height=4cm]{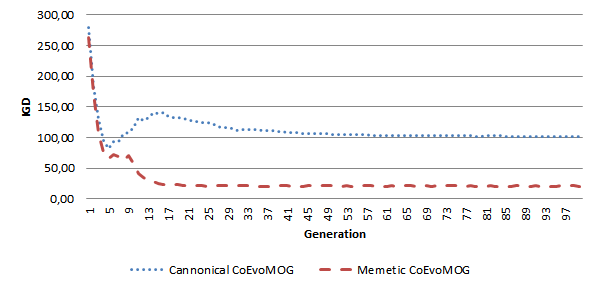}
\caption{Rastrigin 1D}
\end{subfigure}
\begin{subfigure}{0.5\textwidth}
\includegraphics[width=0.99\linewidth, height=4cm]{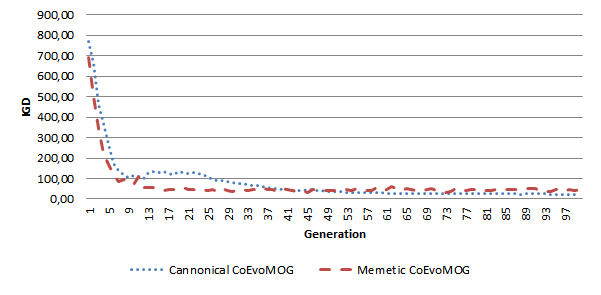}
\caption{Rastrigin 2D}
\end{subfigure}

\begin{subfigure}{0.5\textwidth}
\includegraphics[width=0.99\linewidth, height=4cm]{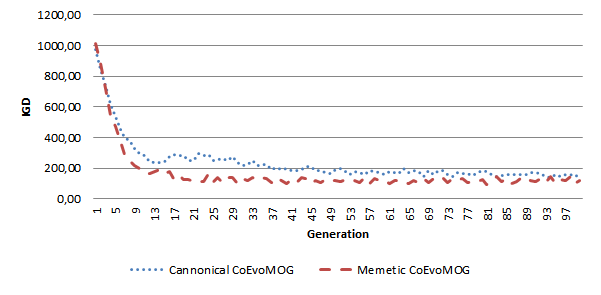}
\caption{Rastrigin 3D}
\end{subfigure}
\begin{subfigure}{0.5\textwidth}
\includegraphics[width=0.99\linewidth, height=4cm]{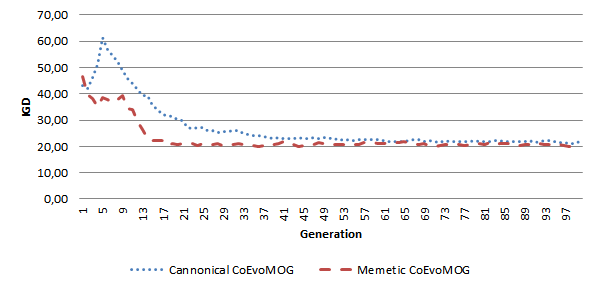}
\caption{Griewank 1D}
\end{subfigure}

\begin{subfigure}{0.5\textwidth}
\includegraphics[width=0.99\linewidth, height=4cm]{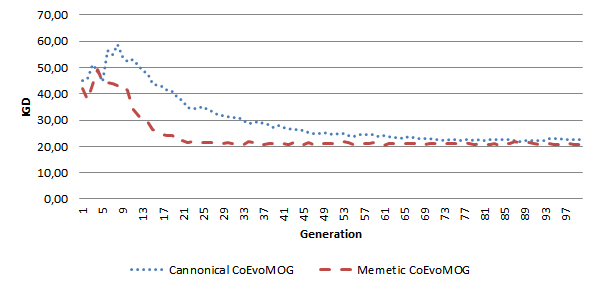}
\caption{Griewank 2D}
\end{subfigure}
\begin{subfigure}{0.5\textwidth}
\includegraphics[width=0.99\linewidth, height=4cm]{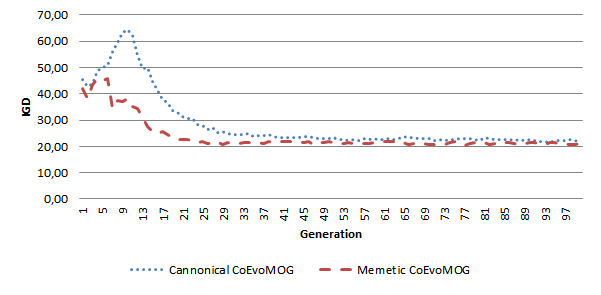}
\caption{Griewank3D}
\end{subfigure}

\caption{Convergence comparison between Canonical and Memetic Co-evolutionary Algorithms for tug-of-war MOG variants with different synthetic functions $\phi$.}
\label{fig:plots}
\end{figure}

\begin{figure}[ht!]
\captionsetup[subfigure]{labelformat=empty}

 \begin{center}
 \begin{subfigure}{0.5\textwidth}
\includegraphics[height=4cm]{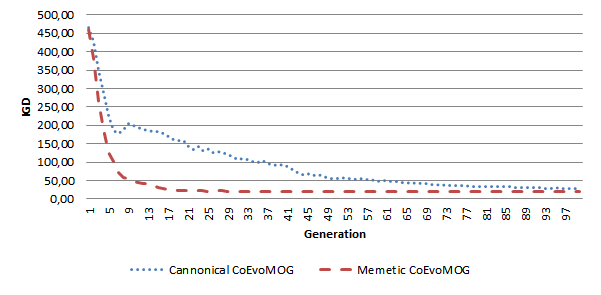}
\caption{Ackley 2D}
\end{subfigure}

\caption{Convergence comparison between Canonical and Memetic Co-evolutionary Algorithms for the tug-of-war MOG variant where $\phi$ = Ackley 2D.}
\label{fig:plots2}
\end{center}
\end{figure}

\subsection{Experimental Results}

The Inverse Generational Distance (IGD) metric was used to measure the convergence characteristics (performance) achieved by the algorithms. IGD combines information about convergence and diversity of the obtained solutions.
%It is possible to use this metric since optimal solution is known.
If $P^{*}$ is a set of uniformly distributed points constituting the Pareto layer, and $F$ is an approximation set of the Pareto layer obtained from the co-evolutionary algorithms, then
$$
IGD = \frac{\sum_{v\in P^{*}}d(v,F)}{|P^{*}|}
$$
where $d(v,F)$ denotes minimum Euclidean distance between $v$ and points in $F$, as measured in the objective space. Clearly, the lower the IGD values the better.

Figures~\ref{fig:plots} and~\ref{fig:plots2} present convergence comparison between Canonical CoEvoMOG and Memetic CoEvoMOG for $9$ tested functions. Plots show the IGD convergence trends averaged over $20$ independent runs. In all cases the Memetic CoEvoMOG algorithm's convergence (dashed line) is noticeably faster than Canonical CoEvoMOG (doted line).

\begin{figure}[ht!]
\captionsetup[subfigure]{labelformat=empty}

\begin{subfigure}{0.32\textwidth}
\includegraphics[width=0.99\linewidth, height=4.2cm]{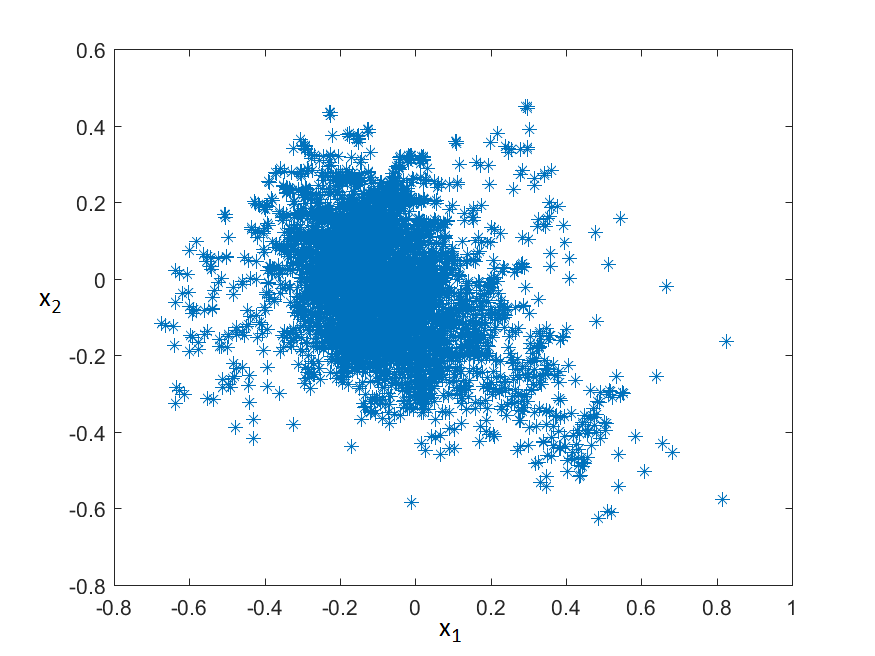}
\caption{Canonical gen. 1}
\end{subfigure}
\begin{subfigure}{0.32\textwidth}
\includegraphics[width=0.99\linewidth, height=4.2cm]{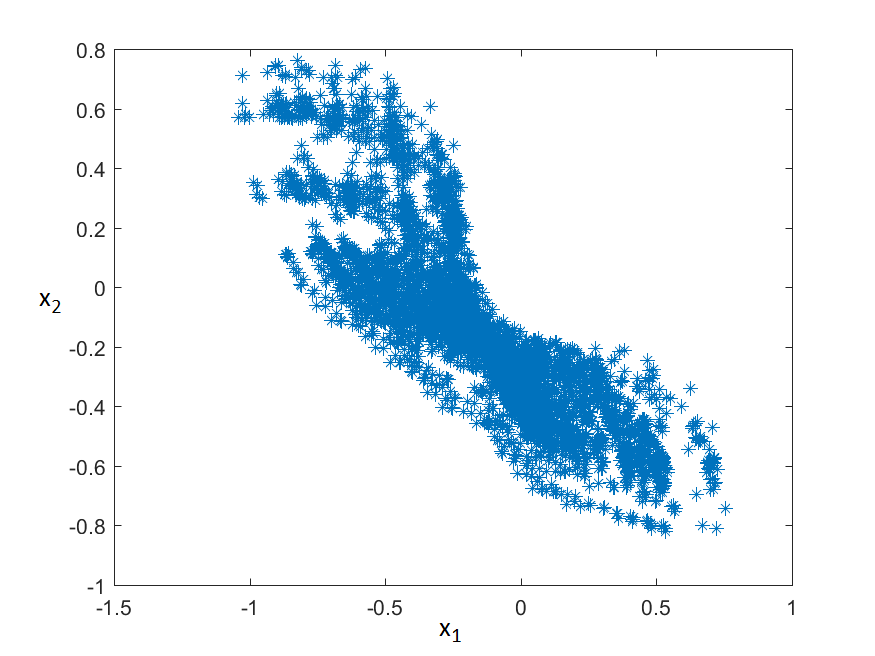}
\caption{Canonical gen. 10}
\end{subfigure}
\begin{subfigure}{0.32\textwidth}
\includegraphics[width=0.99\linewidth, height=4.2cm]{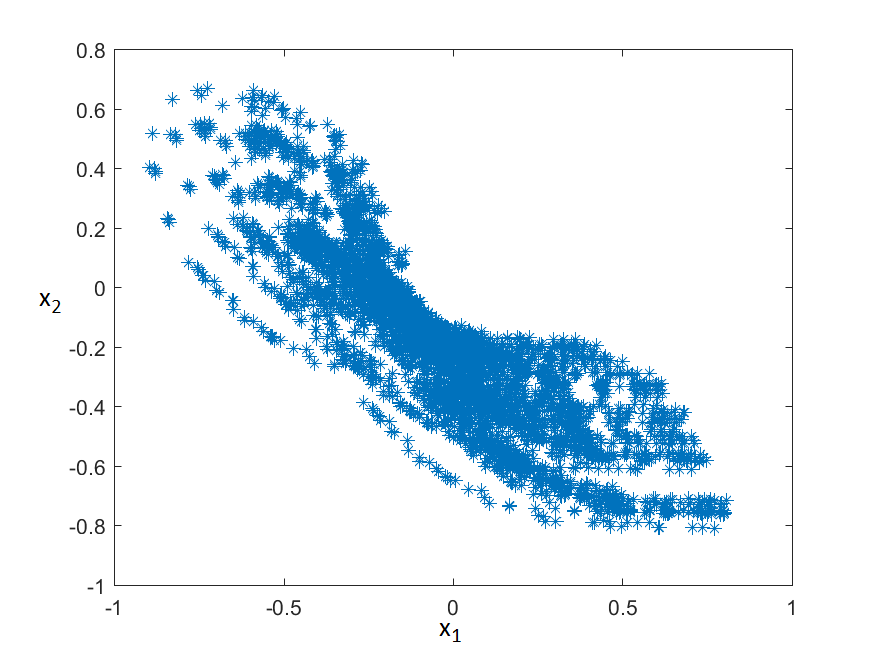}
\caption{Canonical gen. 15}
\end{subfigure}

\begin{subfigure}{0.32\textwidth}
\includegraphics[width=0.99\linewidth, height=4.2cm]{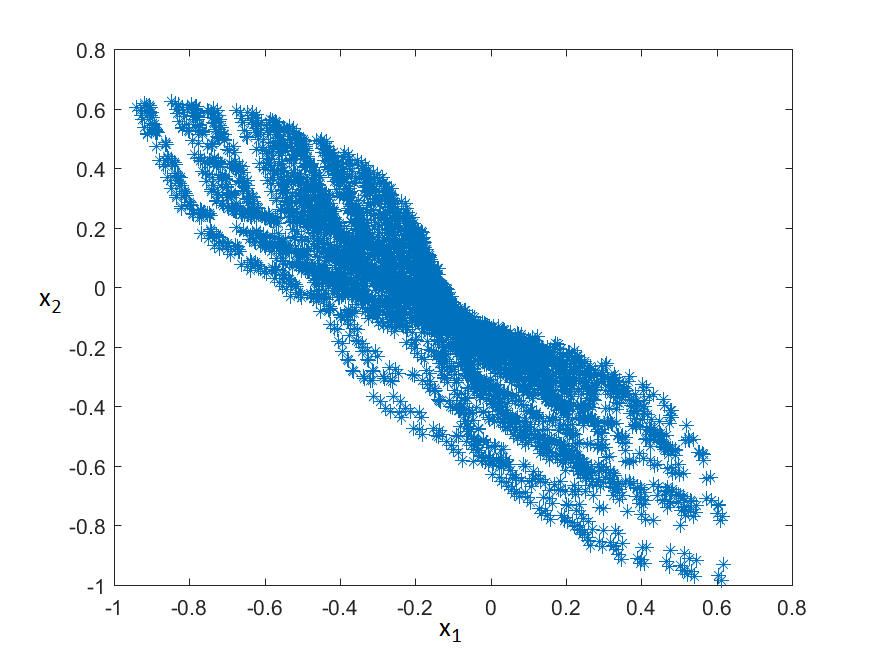}
\caption{Canonical gen. 20}
\end{subfigure}
\begin{subfigure}{0.32\textwidth}
\includegraphics[width=0.99\linewidth, height=4.2cm]{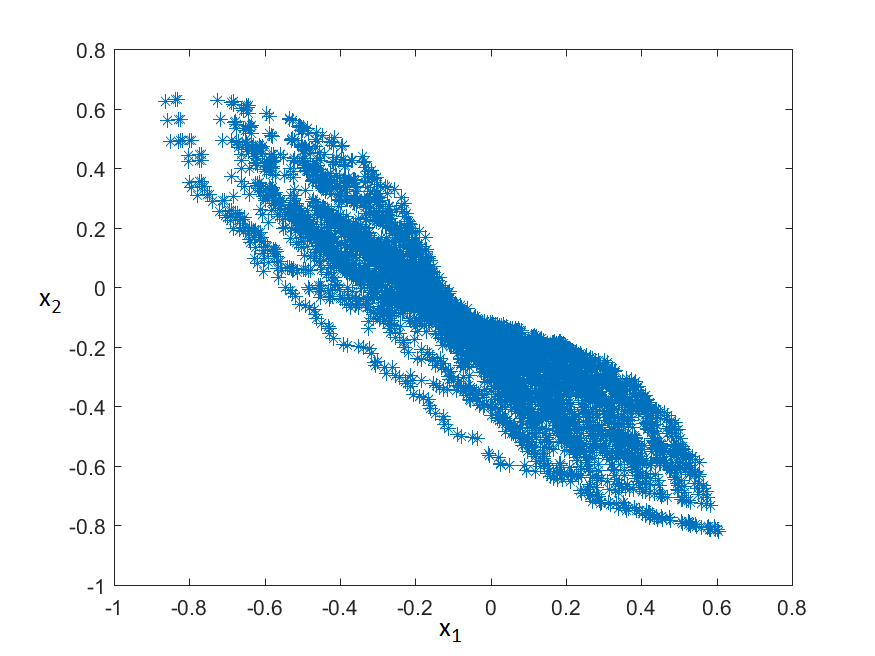}
\caption{Canonical gen. 25}
\end{subfigure}
\begin{subfigure}{0.32\textwidth}
\includegraphics[width=0.99\linewidth, height=4.2cm]{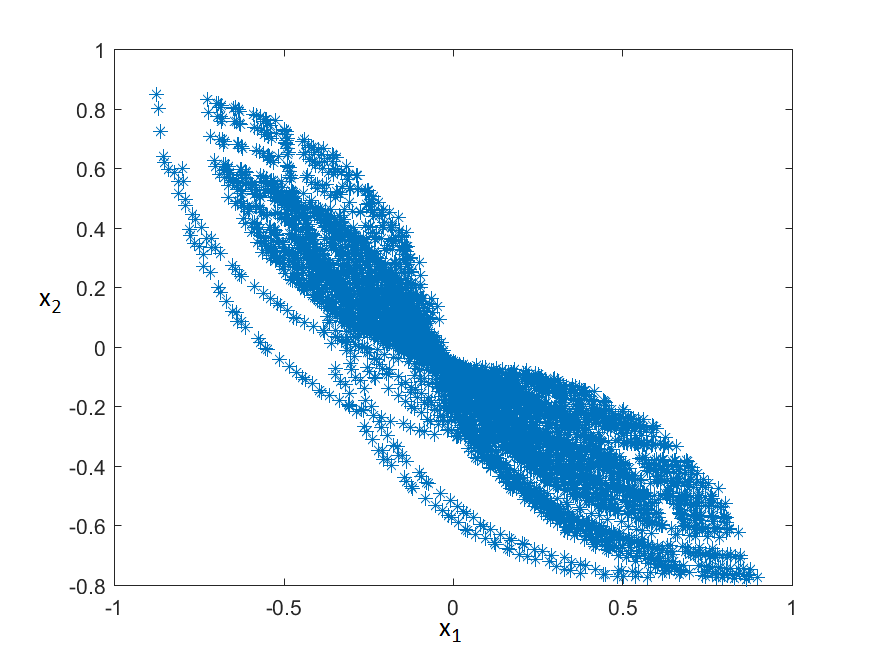}
\caption{Canonical gen. 50}
\end{subfigure}

\begin{subfigure}{0.32\textwidth}
\includegraphics[width=0.99\linewidth, height=4.2cm]{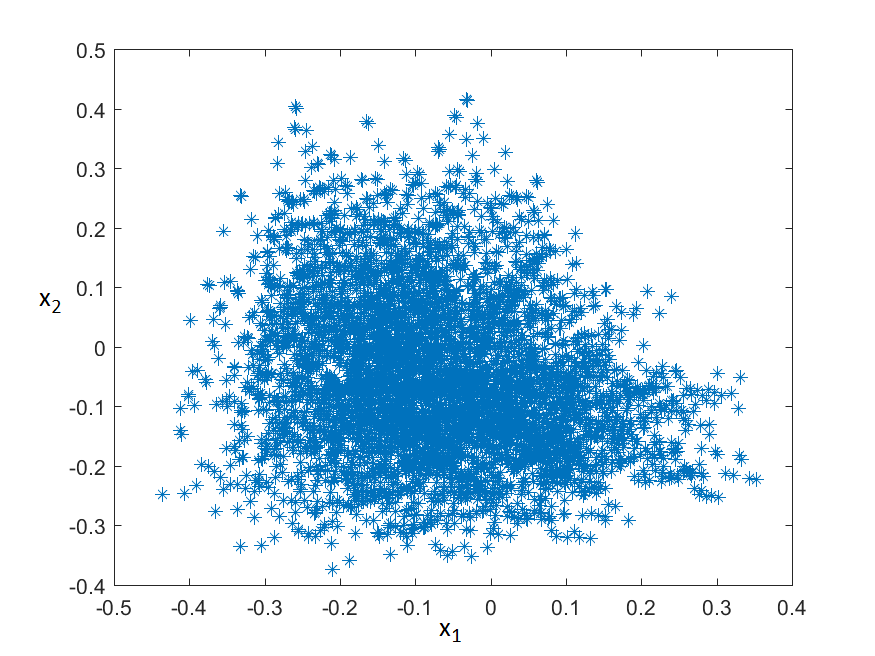}
\caption{Memetic gen. 1}
\end{subfigure}
\begin{subfigure}{0.32\textwidth}
\includegraphics[width=0.99\linewidth, height=4.2cm]{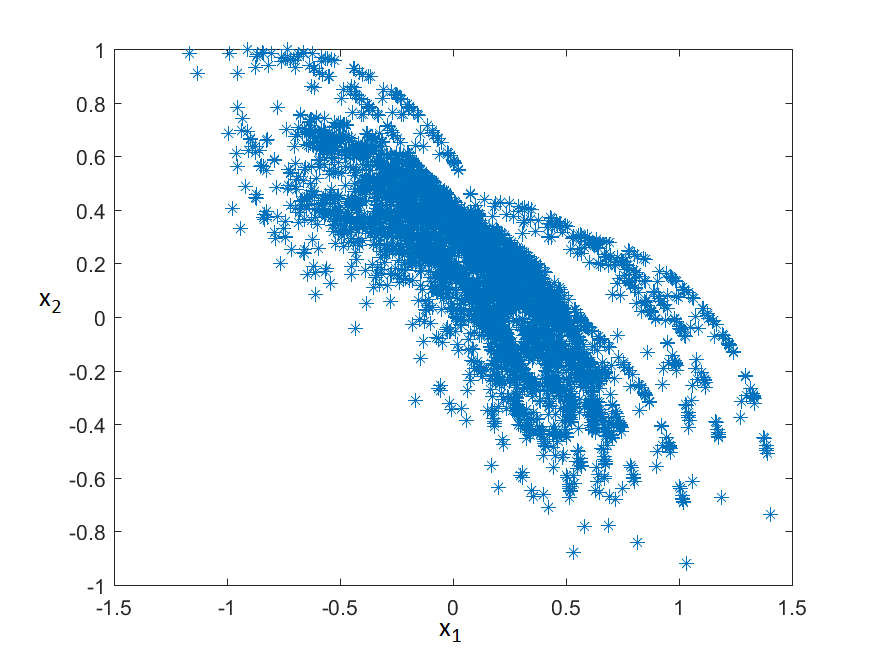}
\caption{Memetic gen. 10}
\end{subfigure}
\begin{subfigure}{0.32\textwidth}
\includegraphics[width=0.99\linewidth, height=4.2cm]{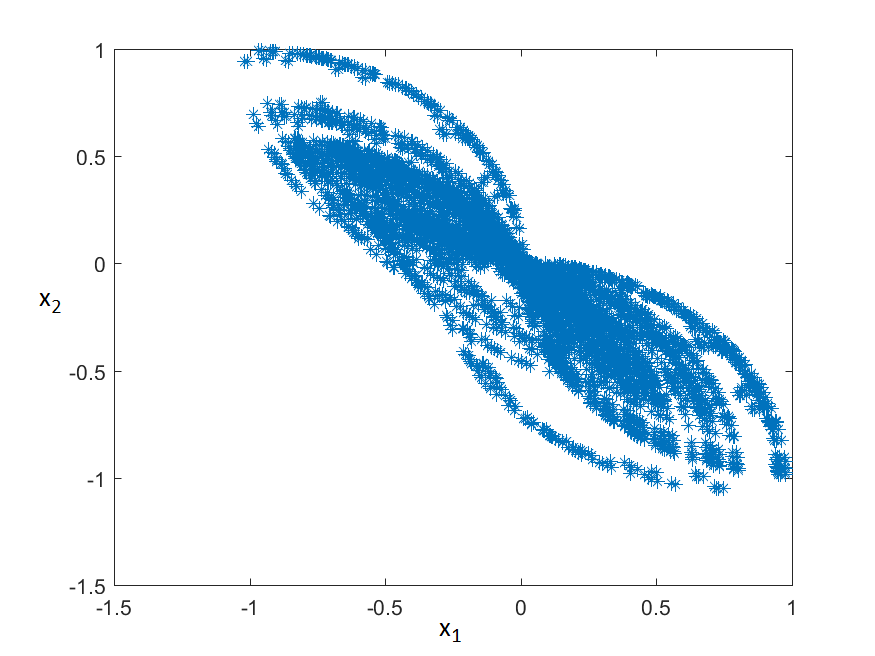}
\caption{Memetic gen. 15}
\end{subfigure}

\begin{subfigure}{0.32\textwidth}
\includegraphics[width=0.99\linewidth, height=4.2cm]{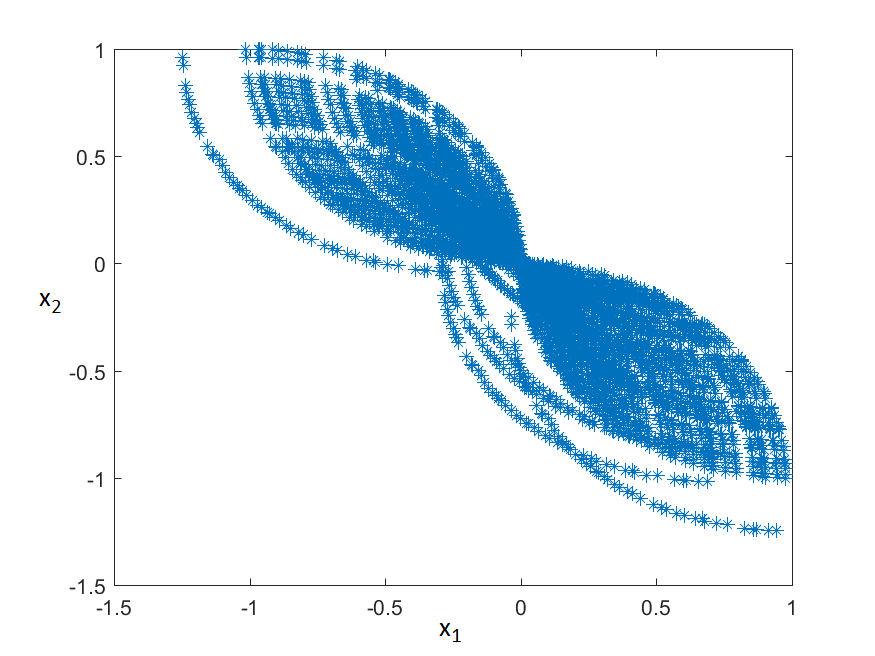}
\caption{Memetic gen. 20}
\end{subfigure}
\begin{subfigure}{0.32\textwidth}
\includegraphics[width=0.99\linewidth, height=4.2cm]{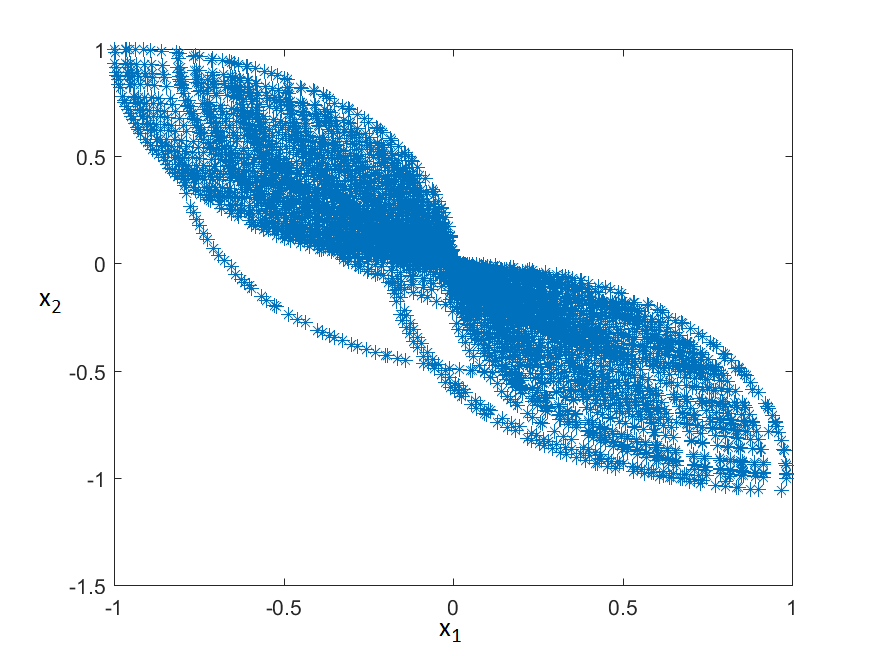}
\caption{Memetic gen. 25}
\end{subfigure}
\begin{subfigure}{0.32\textwidth}
\includegraphics[width=0.99\linewidth, height=4.2cm]{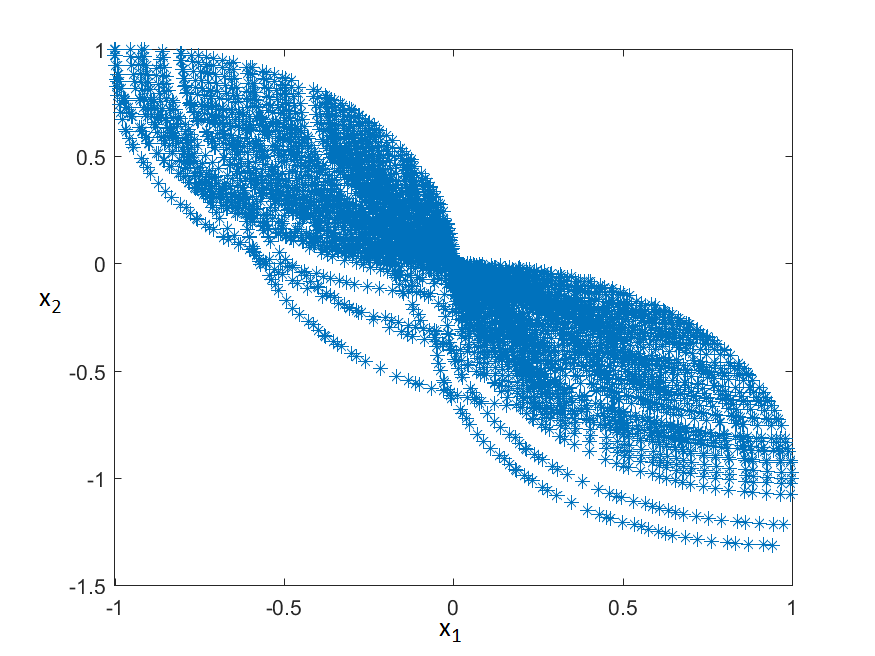}
\caption{Memetic gen. 50}
\end{subfigure}

\caption{The performance of Canonical CoEvoMOG and Memetic CoEvoMOG algorithms on the tug-of-war MOG variant with $\phi$ = Rastrigin 1D, after 1, 10, 15, 20, 25 and 50 generations.}
\label{fig:performanceAckley2D}
\end{figure}

\begin{figure}[ht!]
\captionsetup[subfigure]{labelformat=empty}

\begin{subfigure}{0.32\textwidth}
\includegraphics[width=0.99\linewidth, height=4.2cm]{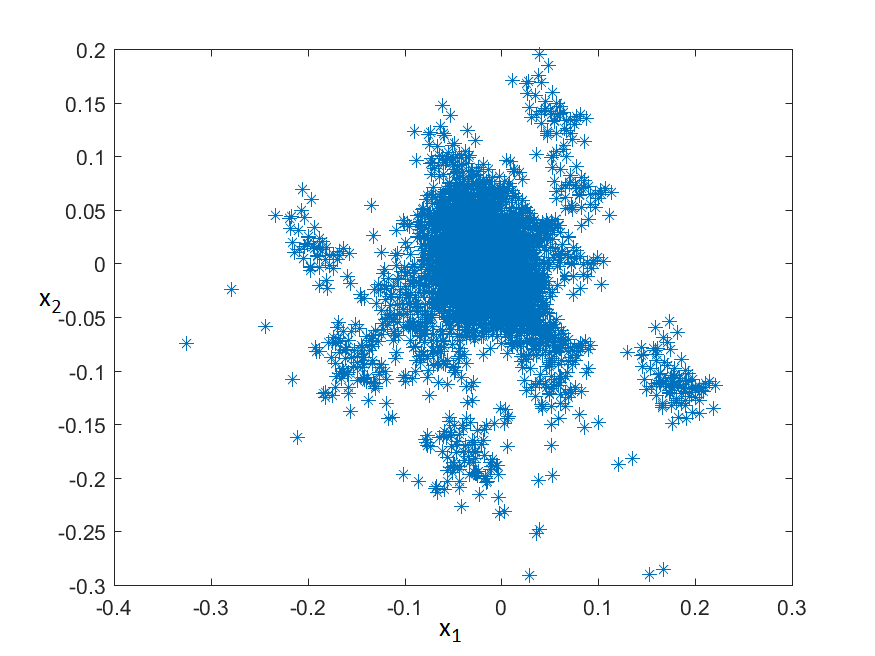}
\caption{Canonical gen. 1}
\end{subfigure}
\begin{subfigure}{0.32\textwidth}
\includegraphics[width=0.99\linewidth, height=4.2cm]{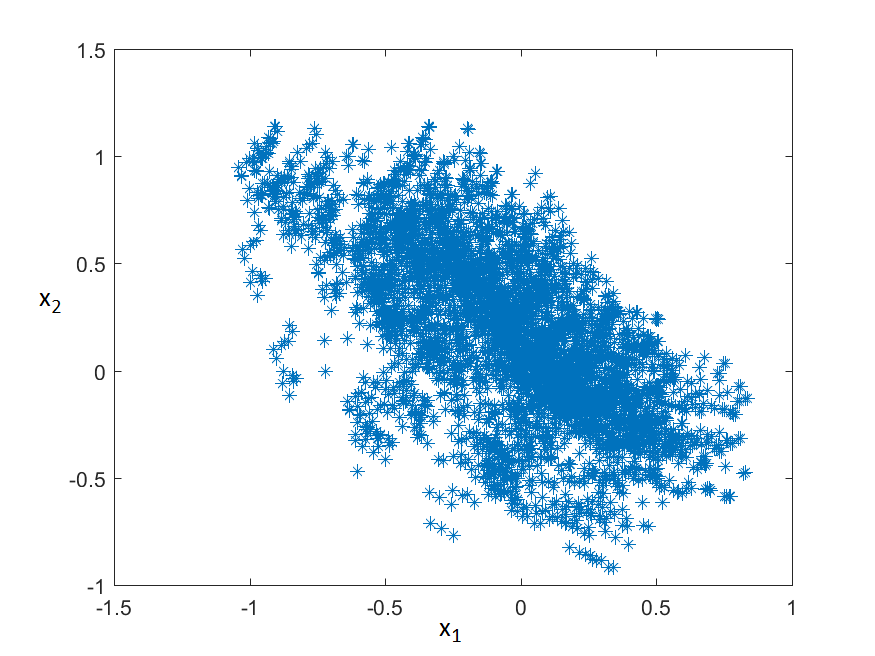}
\caption{Canonical gen. 10}
\end{subfigure}
\begin{subfigure}{0.32\textwidth}
\includegraphics[width=0.99\linewidth, height=4.2cm]{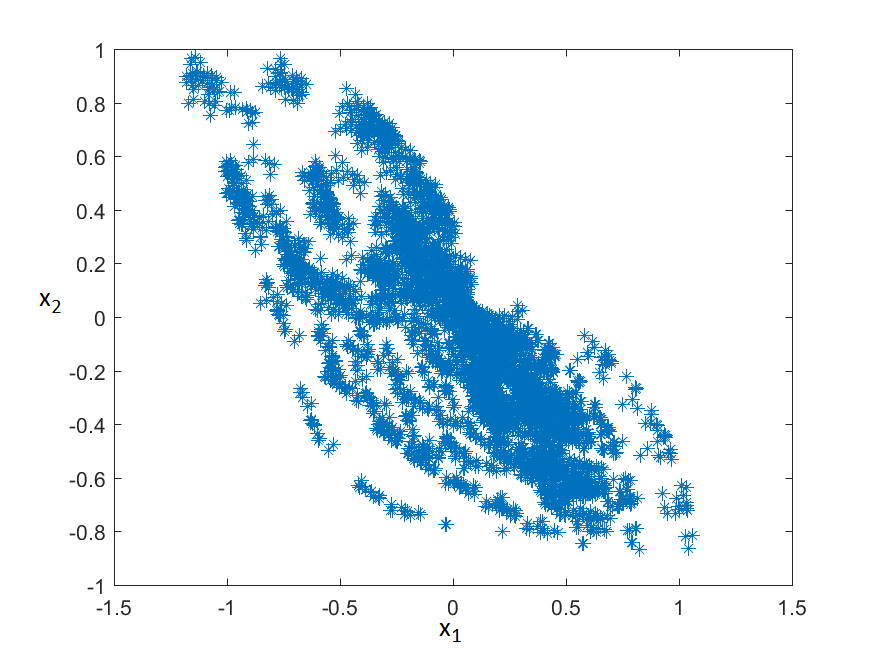}
\caption{Canonical gen. 15}
\end{subfigure}

\begin{subfigure}{0.32\textwidth}
\includegraphics[width=0.99\linewidth, height=4.2cm]{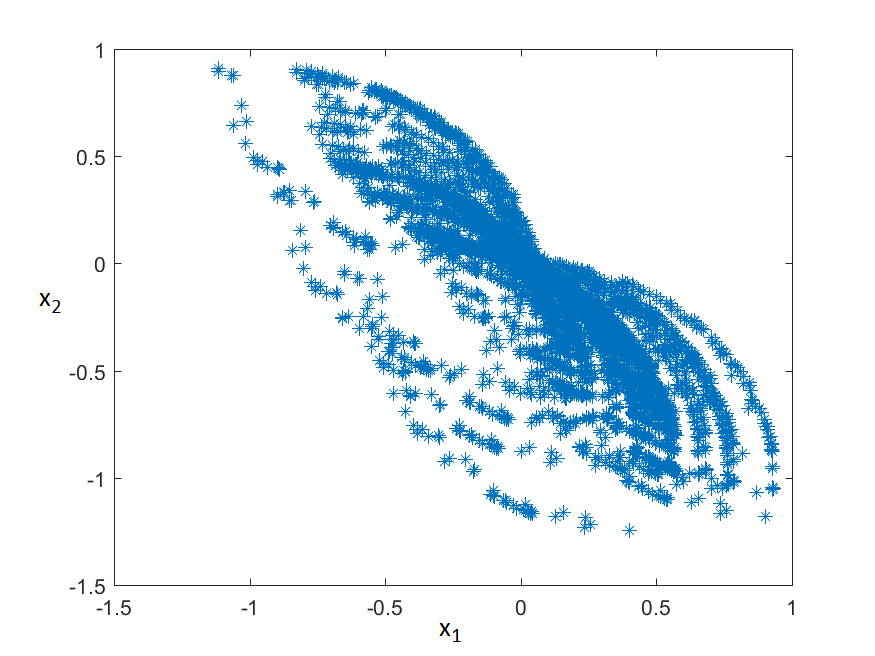}
\caption{Canonical gen. 20}
\end{subfigure}
\begin{subfigure}{0.32\textwidth}
\includegraphics[width=0.99\linewidth, height=4.2cm]{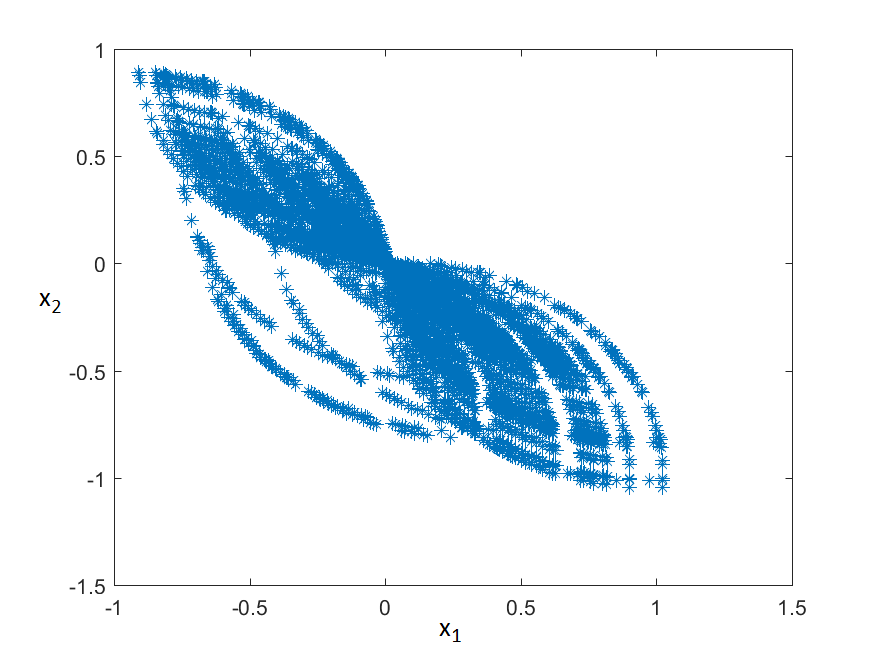}
\caption{Canonical gen. 25}
\end{subfigure}
\begin{subfigure}{0.32\textwidth}
\includegraphics[width=0.99\linewidth, height=4.2cm]{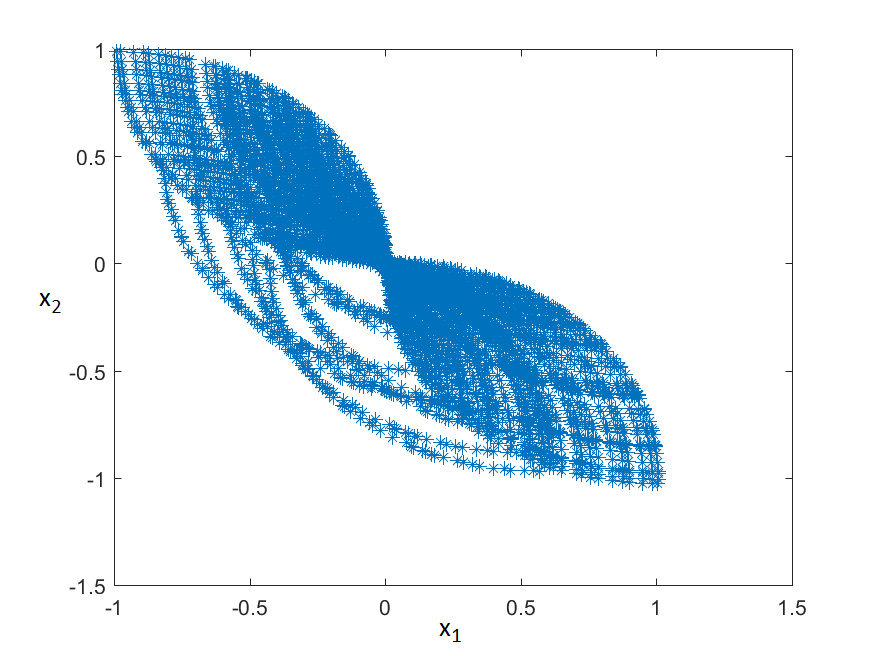}
\caption{Canonical gen. 50}
\end{subfigure}

\begin{subfigure}{0.32\textwidth}
\includegraphics[width=0.99\linewidth, height=4.2cm]{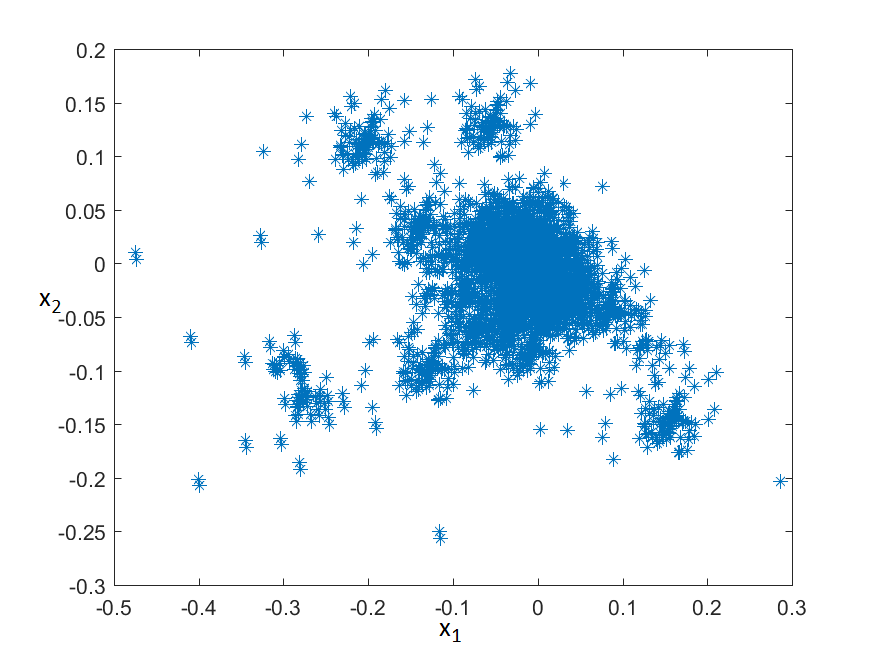}
\caption{Memetic gen. 1}
\end{subfigure}
\begin{subfigure}{0.32\textwidth}
\includegraphics[width=0.99\linewidth, height=4.2cm]{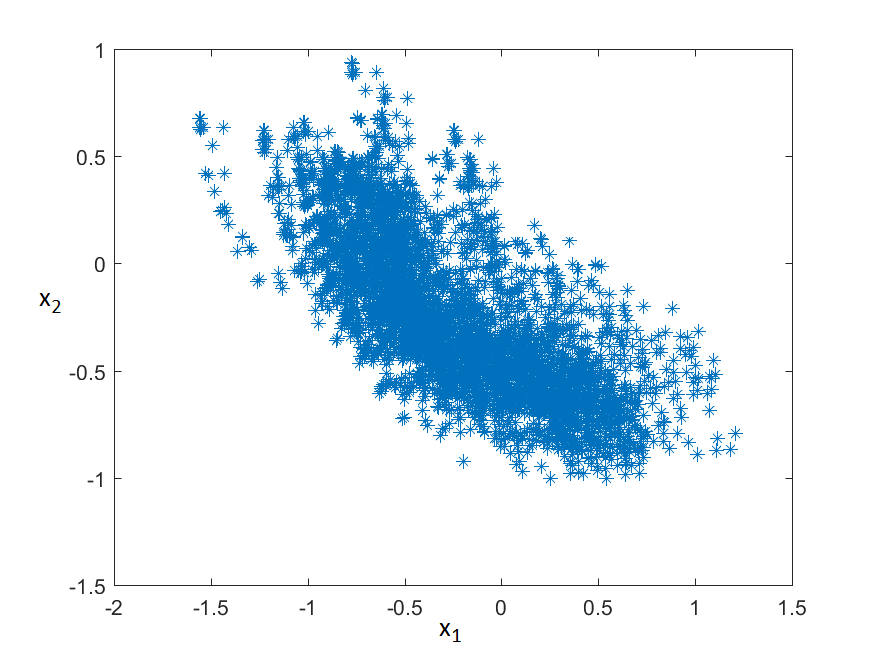}
\caption{Memetic gen. 10}
\end{subfigure}
\begin{subfigure}{0.32\textwidth}
\includegraphics[width=0.99\linewidth, height=4.2cm]{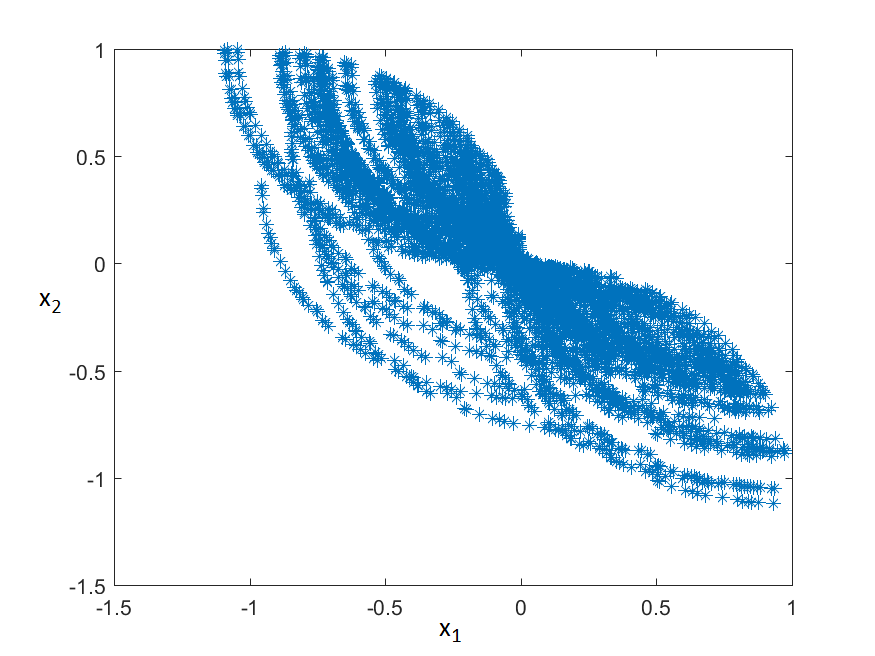}
\caption{Memetic gen. 15}
\end{subfigure}

\begin{subfigure}{0.32\textwidth}
\includegraphics[width=0.99\linewidth, height=4.2cm]{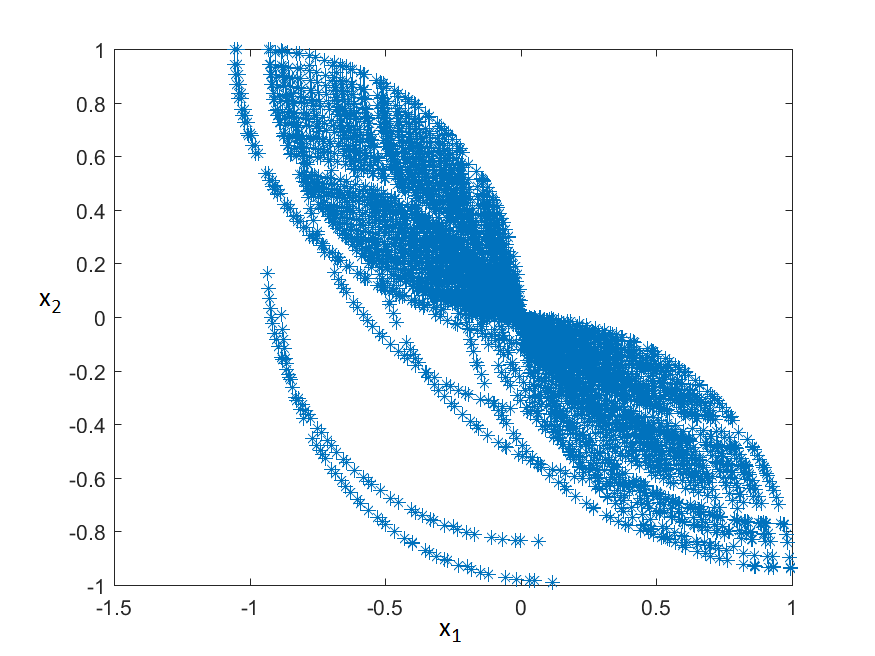}
\caption{Memetic gen. 20}
\end{subfigure}
\begin{subfigure}{0.32\textwidth}
\includegraphics[width=0.99\linewidth, height=4.2cm]{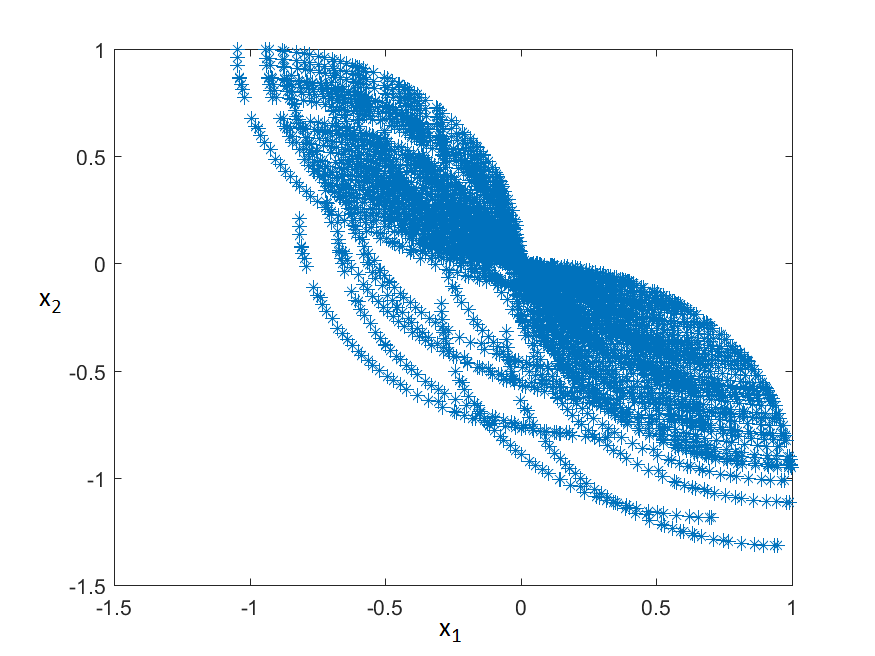}
\caption{Memetic gen. 25}
\end{subfigure}
\begin{subfigure}{0.32\textwidth}
\includegraphics[width=0.99\linewidth, height=4.2cm]{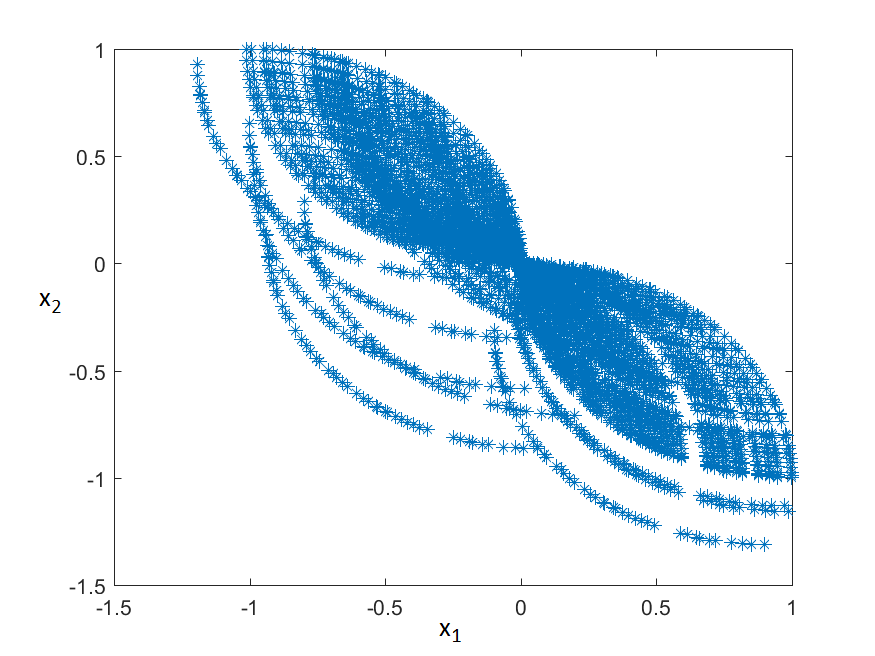}
\caption{Memetic gen. 50}
\end{subfigure}

\caption{The performance of Canonical CoEvoMOG and Memetic CoEvoMOG algorithms on the tug-of-war MOG variant with $\phi$ = Ackley 2D, after 1, 10, 15, 20, 25 and 50 generations.}
\label{fig:performanceRastrigin1D}
\end{figure}

Figures~\ref{fig:performanceAckley2D} and \ref{fig:performanceRastrigin1D} present examples of the algorithms' performance, in the objective space, as the search progresses through the generations. In each plot, each point represents the position of the mass as an outcome of the current strategies contained in the co-evolutionary subpopulations. Refer to Figure~\ref{fig:optimal_solution} for comparing these plots with the a priori known optimal solution (defined by the set of interactions between rationalizable strategies of both players).
%and see how many positions are result of interactions between the rational strategies of both players in each case. 
From generations 1 to 5, both algorithms produce mostly random solutions. The main difference can be noticed to emerge after 10-25 generations when plots of the memetic algorithm are closer to the optimal solution than those of the canonical approach, which agrees with the claim of faster Memetic CoEvoMOG convergence.

Detailed numerical results after every 5 generations for each of the tested MOGs, in terms of average, minimum, maximum and standard deviation of IGD values are presented in the Appendix (Tables~\ref{tab:Rosenbrock2D_results}--\ref{tab:Ackley2D_results}). It can be observed that not only is the convergence faster, but also the final results obtained after $100$ generations are better in the case of Memetic CoEvoMOG. Moreover, the Memetic CoEvoMOG algorithm appears to be more stable - standard deviation values in most of the cases are lower. All results are proved to be statistically significant by the Wilcoxon Signed-Rank Test with p-value=0.05. The numbers of exact function evaluations were identical for both methods.

Slower convergence of the Canonical CoEvoMOG algorithm is hypothesized to be caused by the existence of the Red Queen effect described in Section~\ref{sec:Intrduction}. In most cases, after a few generations of steep decrease of the IGD value, it is observed that the IGD value tends to rise for a brief period of time in the canonical case (as revealed in Figures~\ref{fig:plots} and~\ref{fig:plots2}). This surprising observation may be due to the continuous adaptations of subpopulations to the evolutions of each other, even though the overall performance may be far from the optimum (much like the trees in the forest as discussed in the introduction). In this respect, the local search steps included in the Memetic CoEvoMOG algorithm are seen to provide a guiding hand to the purely genetic mechanisms, thereby suppressing the Red Queen effect to a large extent and accelerating the convergence characteristics of the proposed algorithm.

The above promising experimental results form a strong basis for attempts of solving more complex, real-life problems which can be represented as MOGs. In particular, multi-step decision-making problems and problems characterized by payoffs changing over time seem to be perfect candidates for further evaluation of the Memetic CoEvoMOG algorithm. Such problems appear is various practical domains, including planning and decision-making under uncertainty or in adversarial environments, e.g. in the area of cyber security or homeland security (notably Security Games~\cite{KarwowskiMandziuk2015,KarwowskiMandziuk2016}).
%--------------------------------------------------------------

\section{Conclusions}
\label{sec:conclusions}

This paper presents a new memetic co-evolutionary approach (Memetic CoEvoMOG) to finding strategies for multi-objective games under postponed objective preference articulation. The proposed method improves the canonical co-evolutionary model described in~\cite{Eisenstadt2015} by suppressing the Red Queen effect via the guiding light of lifelong learning. In particular, for ensuring the computational viability of lifelong learning in competitive multi-objective game settings, we incorporate an approach that reduces sets of payoff vectors in objective space to a single representative point without disrupting the postponed preference articulation condition. Thereafter, a surrogate model of the representative point is built, which allows the local improvements to be carried out efficiently by searching on the surrogate landscape. The reliability and effectiveness of our method is experimentally proved on a suite of testing MOG variants. It is demonstrated in the paper that incorporation of memetics improves the convergence characteristics and leads to better solutions in comparison with the canonical co-evolutionary algorithm. Consequently, in the proposed method the total number of function evaluations can be reduced with no harm to the overall quality of resultant strategies. Such time savings are of special importance in the case of time-sensitive and/or computationally expensive MOGs appearing in real-life applications.

Our future research activities shall be concentrated on building upon the current foundations of the Memetic CoEvoMOG algorithms, and extending to several complex multi-step decision-making problems of practical relevance, with particular focus on domains of cyber security and Security Games.

\bibliographystyle{elsarticle-num}

\section{Acknowledgment}
The second author of the paper, Dr. Abhishek Gupta, would like to extend his sincere gratitude to Prof. Amiram Moshaiov, Tel-Aviv University, for valuable discussions on the topic of MOGs that set the foundation for the present work.
This work was partly developed while the third author, Prof. Jacek Ma{\'n}dziuk, was on leave at the School of Computer Science and Engineering, Nanyang Technological University, Singapore.
\bibliography{bibfile}

\section*{Appendix}
\label{sec:appendix}

\begin{table}[ht!]
\scriptsize
%\vspace{-1em}
\begin{center}
\begin{tabular}{|c||r|r|r|r||r|r|r|r|} \cline{1-9}
    & \multicolumn{4}{c||} {CanonicalCoEvoMOG}
    & \multicolumn{4}{c|}{MemeticCoEvoMOG}
\\\hline
  Generation & Avg	& Min & Max	& Std dev	
  & Avg	& Min & Max	& Std dev	
 \\\hline
5 & 65.28 & 39.39 & 133.60 & 30.45 & 54.37 & 43.16 & 110.30 & 20.66 \\
10 & 108.77 & 35.07 & 249.93 & 81.80 & 49.92 & 34.45 & 58.46 & 9.44 \\
15 & 89.17 & 31.81 & 262.70 & 68.31 & 50.89 & 33.31 & 77.97 & 15.41 \\
20 & 72.94 & 27.08 & 244.94 & 63.69 & 43.86 & 28.24 & 95.44 & 19.50 \\
25 & 65.48 & 27.66 & 208.32 & 53.05 & 36.77 & 24.22 & 68.52 & 13.28 \\
30 & 57.10 & 22.71 & 203.50 & 52.95 & 30.68 & 22.28 & 38.74 & 6.32 \\
35 & 50.49 & 22.95 & 198.35 & 52.39 & 28.29 & 23.05 & 36.67 & 4.88 \\
40 & 48.01 & 21.29 & 176.34 & 45.65 & 25.72 & 20.58 & 41.32 & 6.01 \\
45 & 38.29 & 21.80 & 128.73 & 32.03 & 24.13 & 19.85 & 39.28 & 5.72 \\
50 & 37.67 & 20.70 & 102.80 & 24.04 & 23.95 & 20.75 & 36.44 & 4.75 \\
55 & 33.65 & 20.57 & 76.62 & 16.01 & 22.75 & 20.39 & 32.17 & 3.59 \\
60 & 31.71 & 21.75 & 74.82 & 15.75 & 22.57 & 20.48 & 30.74 & 3.41 \\
65 & 31.44 & 22.37 & 65.54 & 13.45 & 22.09 & 20.30 & 26.64 & 2.13 \\
70 & 32.01 & 21.74 & 68.39 & 14.49 & 21.58 & 19.59 & 24.10 & 1.40 \\
75 & 30.76 & 22.99 & 64.73 & 13.26 & 21.26 & 19.20 & 23.57 & 1.13 \\
80 & 30.32 & 21.63 & 64.61 & 13.16 & 20.92 & 18.62 & 22.70 & 1.25 \\
85 & 31.04 & 21.88 & 58.89 & 11.87 & 21.09 & 19.54 & 22.71 & 0.98 \\
90 & 30.44 & 21.64 & 57.39 & 11.35 & 21.07 & 20.11 & 22.79 & 0.76 \\
95 & 29.24 & 20.99 & 60.32 & 11.61 & 20.89 & 20.00 & 21.73 & 0.55 \\
100 & 29.45 & 20.82 & 58.14 & 11.36 & 20.96 & 19.72 & 22.89 & 0.91 \\
\hline
\end{tabular}
\end{center}
\caption{Comparison between results (in terms of IGD) obtained by Cannonical and Memetic Co-evolutionary Algorithms based on 20 independent runs of the tug-of-war MOG variant with $\phi$ = Rosenbrock 2D.}
\label{tab:Rosenbrock2D_results}
\end{table}

\begin{table}[ht!]
\scriptsize
%\vspace{-1em}
\begin{center}
\begin{tabular}{|c||r|r|r|r||r|r|r|r|} \cline{1-9}
    & \multicolumn{4}{c||} {CanonicalCoEvoMOG}
    & \multicolumn{4}{c|}{MemeticCoEvoMOG}
\\\hline
  Generation & Avg	& Min & Max	& Std dev	
  & Avg	& Min & Max	& Std dev	
 \\\hline
5 & 221.43 & 62.10 & 422.71 & 121.20 & 159.70 & 98.87 & 210.68 & 30.68 \\
10 & 139.15 & 94.28 & 235.24 & 40.96 & 127.09 & 65.20 & 344.80 & 80.01 \\
15 & 163.41 & 85.68 & 348.19 & 77.72 & 96.67 & 40.44 & 177.86 & 41.80 \\
20 & 142.60 & 40.23 & 333.12 & 85.41 & 107.87 & 55.31 & 232.53 & 56.76 \\
25 & 156.23 & 34.67 & 325.21 & 83.47 & 88.77 & 44.01 & 169.74 & 42.91 \\
30 & 156.85 & 37.83 & 392.68 & 99.68 & 91.32 & 49.53 & 240.98 & 55.93 \\
35 & 148.23 & 31.43 & 388.38 & 98.04 & 81.51 & 37.89 & 206.72 & 50.32 \\
40 & 133.30 & 33.33 & 334.02 & 91.26 & 74.29 & 38.13 & 193.34 & 44.43 \\
45 & 140.77 & 32.21 & 356.72 & 97.36 & 67.55 & 34.48 & 164.75 & 37.58 \\
50 & 130.92 & 30.93 & 356.50 & 95.12 & 53.79 & 32.19 & 126.93 & 27.58 \\
55 & 131.97 & 27.86 & 359.18 & 98.65 & 47.50 & 33.15 & 108.08 & 23.51 \\
60 & 123.37 & 26.57 & 351.25 & 95.70 & 45.11 & 28.85 & 86.73 & 16.29 \\
65 & 118.80 & 26.21 & 345.35 & 93.30 & 41.47 & 26.04 & 84.47 & 16.83 \\
70 & 121.05 & 26.22 & 401.82 & 110.77 & 39.08 & 24.33 & 73.32 & 15.01 \\
75 & 108.20 & 27.95 & 292.76 & 81.13 & 38.25 & 26.62 & 86.44 & 17.80 \\
80 & 112.03 & 26.12 & 302.42 & 82.19 & 38.89 & 22.26 & 88.13 & 20.71 \\
85 & 110.11 & 25.34 & 302.54 & 81.86 & 36.36 & 22.45 & 79.01 & 17.17 \\
90 & 109.76 & 24.38 & 325.80 & 85.75 & 32.38 & 22.31 & 61.75 & 12.70 \\
95 & 116.73 & 23.13 & 319.10 & 84.37 & 32.46 & 22.90 & 59.00 & 11.87 \\
100 & 112.51 & 22.08 & 318.41 & 84.95 & 31.13 & 22.93 & 47.49 & 7.96 \\
\hline
\end{tabular}
\end{center}
\caption{Comparison between results (in terms of IGD) obtained by Cannonical and Memetic Co-evolutionary Algorithms based on 20 independent runs of the tug-of-war MOG variant with $\phi$ = Rosenbrock 3D.}
\label{tab:Rosenbrock3D_results}
\end{table}

\begin{table}[ht!]
\scriptsize
%\vspace{-1em}
\begin{center}
\begin{tabular}{|c||r|r|r|r||r|r|r|r|} \cline{1-9}
    & \multicolumn{4}{c||} {CanonicalCoEvoMOG}
    & \multicolumn{4}{c|}{MemeticCoEvoMOG}
\\\hline
  Generation & Avg	& Min & Max	& Std dev	
  & Avg	& Min & Max	& Std dev	
 \\\hline
5 & 82.69 & 35.23 & 182.31 & 51.71 & 65.59 & 34.88 & 128.11 & 36.63 \\
10 & 115.71 & 40.70 & 532.32 & 149.59 & 57.21 & 28.04 & 100.62 & 26.31 \\
15 & 138.51 & 31.11 & 822.00 & 241.86 & 25.37 & 19.83 & 37.58 & 5.31 \\
20 & 131.23 & 24.76 & 858.55 & 257.37 & 22.61 & 20.40 & 30.55 & 3.02 \\
25 & 124.50 & 25.75 & 843.10 & 253.38 & 21.04 & 19.59 & 22.69 & 1.11 \\
30 & 114.05 & 24.59 & 820.08 & 248.31 & 21.37 & 20.13 & 22.45 & 0.76 \\
35 & 112.73 & 22.83 & 821.77 & 249.36 & 21.09 & 20.27 & 22.50 & 0.80 \\
40 & 109.67 & 21.45 & 821.12 & 250.10 & 21.43 & 20.36 & 22.84 & 0.75 \\
45 & 106.52 & 22.63 & 820.33 & 250.87 & 21.20 & 19.02 & 22.45 & 1.02 \\
50 & 104.97 & 20.45 & 819.34 & 251.06 & 21.74 & 19.67 & 23.46 & 0.96 \\
55 & 104.66 & 20.80 & 819.51 & 251.20 & 20.88 & 19.91 & 22.35 & 0.64 \\
60 & 103.91 & 21.29 & 820.20 & 251.69 & 20.99 & 19.35 & 22.76 & 1.05 \\
65 & 103.19 & 20.91 & 819.72 & 251.77 & 21.34 & 19.87 & 22.10 & 0.64 \\
70 & 103.65 & 22.21 & 819.06 & 251.37 & 20.97 & 19.91 & 21.90 & 0.75 \\
75 & 102.98 & 20.52 & 819.22 & 251.66 & 21.38 & 20.20 & 22.76 & 0.85 \\
80 & 102.24 & 19.00 & 817.93 & 251.48 & 21.18 & 19.12 & 22.76 & 1.08 \\
85 & 102.15 & 18.99 & 819.12 & 251.93 & 20.72 & 19.22 & 21.88 & 0.95 \\
90 & 102.29 & 21.44 & 818.64 & 251.70 & 20.95 & 19.94 & 22.48 & 0.79 \\
95 & 101.57 & 20.84 & 818.62 & 251.95 & 21.11 & 19.84 & 23.26 & 1.00 \\
100 & 101.78 & 21.07 & 818.19 & 251.72 & 21.00 & 18.74 & 22.56 & 1.26 \\

\hline
\end{tabular}
\end{center}
\caption{Comparison between results (in terms of IGD) obtained by Cannonical and Memetic Co-evolutionary Algorithms based on 20 independent runs of the tug-of-war MOG variant with $\phi$ = Rastrigin 1D.}
\label{tab:Rastrigin1D_results}
\end{table}

\begin{table}[ht!]
\scriptsize
%\vspace{-1em}
\begin{center}
\begin{tabular}{|c||r|r|r|r||r|r|r|r|} \cline{1-9}
    & \multicolumn{4}{c||} {CanonicalCoEvoMOG}
    & \multicolumn{4}{c|}{MemeticCoEvoMOG}
\\\hline
  Generation & Avg	& Min & Max	& Std dev	
  & Avg	& Min & Max	& Std dev	
 \\\hline
5 & 267.25 & 105.65 & 401.83 & 97.01 & 167.62 & 62.95 & 347.25 & 79.07 \\
10 & 112.35 & 30.70 & 320.32 & 96.15 & 72.66 & 31.69 & 240.47 & 69.85 \\
15 & 129.75 & 37.76 & 424.67 & 126.87 & 47.83 & 21.72 & 194.51 & 52.01 \\
20 & 130.90 & 35.54 & 492.63 & 142.76 & 58.24 & 20.34 & 369.19 & 109.29 \\
25 & 108.05 & 24.44 & 374.31 & 117.24 & 42.87 & 20.19 & 230.27 & 65.86 \\
30 & 81.50 & 21.98 & 306.10 & 87.68 & 36.85 & 21.13 & 170.69 & 47.03 \\
35 & 64.85 & 21.36 & 195.60 & 58.29 & 44.76 & 19.40 & 251.54 & 72.67 \\
40 & 48.32 & 20.80 & 121.05 & 33.20 & 51.40 & 20.40 & 322.40 & 95.23 \\
45 & 44.62 & 21.74 & 97.24 & 28.69 & 33.58 & 20.11 & 141.78 & 38.03 \\
50 & 36.43 & 21.50 & 76.02 & 17.43 & 41.13 & 19.03 & 219.55 & 62.71 \\
55 & 32.11 & 21.30 & 60.89 & 13.70 & 52.31 & 20.09 & 330.97 & 97.91 \\
60 & 29.49 & 20.61 & 50.33 & 10.26 & 42.92 & 20.11 & 237.23 & 68.28 \\
65 & 28.88 & 21.08 & 46.79 & 9.63 & 50.44 & 20.18 & 312.75 & 92.17 \\
70 & 26.92 & 21.45 & 39.10 & 6.44 & 52.59 & 20.12 & 336.27 & 99.68 \\
75 & 26.61 & 20.68 & 40.84 & 7.53 & 49.01 & 19.74 & 299.81 & 88.13 \\
80 & 26.55 & 21.64 & 42.22 & 6.31 & 36.21 & 20.23 & 168.79 & 46.59 \\
85 & 25.01 & 21.43 & 33.32 & 3.94 & 48.01 & 20.05 & 281.77 & 82.14 \\
90 & 24.29 & 21.40 & 28.88 & 2.93 & 49.40 & 20.32 & 302.58 & 88.96 \\
95 & 23.86 & 19.51 & 30.49 & 3.28 & 46.31 & 20.26 & 267.55 & 77.74 \\
100 & 23.72 & 20.77 & 29.40 & 3.01 & 47.22 & 19.79 & 283.01 & 82.86 \\
\hline
\end{tabular}
\end{center}
\caption{Comparison between results (in terms of IGD) obtained by Cannonical and Memetic Co-evolutionary Algorithms based on 20 independent runs of the tug-of-war MOG variant with $\phi$ = Rastrigin 2D.}
\label{tab:Rastrigin2D_results}
\end{table}

\begin{table}[ht!]
\scriptsize
%\vspace{-1em}
\begin{center}
\begin{tabular}{|c||r|r|r|r||r|r|r|r|} \cline{1-9}
    & \multicolumn{4}{c||} {CanonicalCoEvoMOG}
    & \multicolumn{4}{c|}{MemeticCoEvoMOG}
\\\hline
  Generation & Avg	& Min & Max	& Std dev	
  & Avg	& Min & Max	& Std dev	
 \\\hline
5 & 539.35 & 264.70 & 793.37 & 170.54 & 473.73 & 257.24 & 708.45 & 122.06 \\
10 & 297.53 & 109.05 & 497.63 & 122.76 & 195.57 & 28.67 & 559.76 & 172.42 \\
15 & 246.47 & 47.36 & 574.07 & 177.01 & 164.72 & 42.34 & 483.88 & 175.39 \\
20 & 253.62 & 45.83 & 514.96 & 172.03 & 128.32 & 29.95 & 538.47 & 168.63 \\
25 & 253.72 & 43.73 & 559.89 & 182.96 & 115.86 & 22.74 & 404.96 & 146.92 \\
30 & 239.35 & 32.67 & 555.56 & 185.00 & 103.92 & 20.35 & 362.89 & 135.21 \\
35 & 226.37 & 24.05 & 520.98 & 187.45 & 140.85 & 19.90 & 585.08 & 205.77 \\
40 & 200.78 & 25.86 & 570.16 & 185.05 & 100.02 & 19.67 & 359.16 & 133.79 \\
45 & 208.77 & 24.15 & 487.22 & 185.19 & 101.88 & 19.90 & 381.89 & 134.46 \\
50 & 189.19 & 23.75 & 483.84 & 180.69 & 121.58 & 20.75 & 616.96 & 193.50 \\
55 & 171.68 & 23.93 & 431.40 & 170.07 & 106.75 & 20.52 & 350.30 & 138.17 \\
60 & 161.08 & 23.21 & 460.39 & 161.00 & 119.49 & 20.42 & 568.65 & 182.07 \\
65 & 170.20 & 22.22 & 499.57 & 182.10 & 100.48 & 20.56 & 452.16 & 144.64 \\
70 & 161.04 & 23.58 & 464.79 & 171.77 & 131.46 & 20.61 & 488.36 & 184.60 \\
75 & 170.60 & 23.12 & 455.10 & 183.43 & 129.83 & 20.16 & 427.89 & 178.37 \\
80 & 183.36 & 22.01 & 549.62 & 209.91 & 125.35 & 19.79 & 560.23 & 185.42 \\
85 & 156.82 & 22.85 & 400.55 & 173.16 & 130.58 & 21.33 & 566.23 & 190.19 \\
90 & 170.47 & 21.48 & 450.96 & 192.60 & 118.76 & 20.29 & 446.60 & 161.69 \\
95 & 158.73 & 21.84 & 434.21 & 176.66 & 99.18 & 20.16 & 340.70 & 127.85 \\
100 & 162.51 & 20.60 & 492.29 & 188.31 & 121.98 & 20.07 & 433.02 & 165.82 \\
\hline
\end{tabular}
\end{center}
\caption{Comparison between results (in terms of IGD) obtained by Cannonical and Memetic Co-evolutionary Algorithms based on 20 independent runs of the tug-of-war MOG variant with $\phi$ = Rastrigin 3D.}
\label{tab:Rastrigin3D_results}
\end{table}

\begin{table}[ht!]
\scriptsize
%\vspace{-1em}
\begin{center}
\begin{tabular}{|c||r|r|r|r||r|r|r|r|} \cline{1-9}
    & \multicolumn{4}{c||} {CanonicalCoEvoMOG}
    & \multicolumn{4}{c|}{MemeticCoEvoMOG}
\\\hline
  Generation & Avg	& Min & Max	& Std dev	
  & Avg	& Min & Max	& Std dev	
 \\\hline
5 & 61.50 & 28.44 & 145.25 & 35.46 & 38.47 & 30.15 & 83.04 & 15.51 \\
10 & 45.79 & 22.07 & 110.52 & 25.45 & 34.29 & 20.95 & 46.41 & 7.90 \\
15 & 35.56 & 21.10 & 85.59 & 20.25 & 22.36 & 20.52 & 28.91 & 2.58 \\
20 & 30.16 & 22.15 & 62.02 & 11.87 & 20.61 & 20.21 & 23.69 & 1.02 \\
25 & 27.06 & 21.44 & 47.33 & 7.92 & 20.93 & 19.66 & 22.52 & 1.03 \\
30 & 25.77 & 21.85 & 40.16 & 5.61 & 20.85 & 19.59 & 22.33 & 0.83 \\
35 & 24.12 & 20.17 & 35.10 & 4.53 & 20.38 & 20.12 & 22.01 & 0.67 \\
40 & 23.20 & 20.67 & 31.75 & 3.34 & 20.99 & 19.47 & 22.19 & 0.92 \\
45 & 23.05 & 20.51 & 30.91 & 2.90 & 20.51 & 19.69 & 23.65 & 1.36 \\
50 & 23.39 & 21.29 & 28.48 & 2.18 & 21.15 & 19.62 & 21.89 & 0.89 \\
55 & 22.20 & 20.43 & 25.15 & 1.57 & 20.55 & 19.51 & 22.29 & 0.93 \\
60 & 22.57 & 21.01 & 25.76 & 1.68 & 20.92 & 19.79 & 22.49 & 0.88 \\
65 & 21.98 & 20.64 & 24.91 & 1.38 & 21.93 & 20.73 & 22.90 & 0.75 \\
70 & 22.25 & 20.61 & 23.89 & 1.19 & 20.33 & 19.75 & 23.45 & 1.16 \\
75 & 21.98 & 20.42 & 23.34 & 0.83 & 21.00 & 19.38 & 22.49 & 1.00 \\
80 & 21.94 & 19.51 & 23.72 & 1.20 & 20.96 & 19.72 & 21.81 & 0.71 \\
85 & 21.94 & 20.63 & 23.16 & 0.79 & 21.02 & 20.18 & 22.81 & 0.85 \\
90 & 22.08 & 20.66 & 24.07 & 0.94 & 20.84 & 19.90 & 21.68 & 0.59 \\
95 & 21.66 & 20.02 & 23.56 & 1.06 & 20.39 & 19.71 & 22.63 & 0.98 \\
100 & 21.65 & 20.10 & 23.28 & 1.12 & 20.75 & 19.86 & 21.80 & 0.69 \\
\hline
\end{tabular}
\end{center}
\caption{Comparison between results (in terms of IGD) obtained by Cannonical and Memetic Co-evolutionary Algorithms based on 20 independent runs of the tug-of-war MOG variant with $\phi$ = Griewank 1D.}
\label{tab:Griewank1D_results}
\end{table}

\begin{table}[ht!]
\scriptsize
%\vspace{-1em}
\begin{center}
\begin{tabular}{|c||r|r|r|r||r|r|r|r|} \cline{1-9}
    & \multicolumn{4}{c||} {CanonicalCoEvoMOG}
    & \multicolumn{4}{c|}{MemeticCoEvoMOG}
\\\hline
  Generation & Avg	& Min & Max	& Std dev	
  & Avg	& Min & Max	& Std dev	
 \\\hline
5 & 44.79 & 29.18 & 72.61 & 14.84 & 44.64 & 28.02 & 72.71 & 13.84 \\
10 & 52.15 & 42.66 & 69.53 & 10.57 & 41.19 & 30.52 & 50.83 & 6.56 \\
15 & 43.86 & 25.76 & 74.00 & 15.32 & 26.42 & 21.14 & 37.49 & 4.74 \\
20 & 38.64 & 24.60 & 64.01 & 13.00 & 22.03 & 20.17 & 24.37 & 1.41 \\
25 & 35.14 & 23.91 & 58.30 & 10.39 & 21.33 & 20.24 & 24.37 & 1.14 \\
30 & 31.45 & 22.79 & 47.51 & 7.35 & 21.30 & 19.44 & 22.20 & 0.84 \\
35 & 28.89 & 22.49 & 44.95 & 6.76 & 21.28 & 19.82 & 21.95 & 0.69 \\
40 & 27.76 & 21.46 & 45.63 & 7.11 & 21.42 & 19.47 & 22.31 & 0.94 \\
45 & 26.20 & 22.20 & 35.58 & 3.93 & 20.88 & 19.87 & 22.15 & 0.73 \\
50 & 24.90 & 21.05 & 33.30 & 3.43 & 21.11 & 19.47 & 24.04 & 1.27 \\
55 & 23.89 & 20.63 & 32.65 & 3.35 & 20.72 & 19.76 & 22.16 & 0.67 \\
60 & 23.57 & 21.02 & 31.07 & 2.94 & 21.17 & 20.11 & 22.59 & 0.76 \\
65 & 23.06 & 21.07 & 27.99 & 2.06 & 21.17 & 19.45 & 22.16 & 0.79 \\
70 & 23.00 & 20.68 & 27.28 & 1.89 & 21.05 & 20.10 & 21.96 & 0.60 \\
75 & 22.66 & 20.98 & 26.23 & 1.64 & 21.10 & 20.27 & 22.10 & 0.66 \\
80 & 22.75 & 21.29 & 23.93 & 0.79 & 21.25 & 20.72 & 22.08 & 0.45 \\
85 & 22.56 & 20.31 & 25.34 & 1.61 & 21.01 & 18.76 & 22.23 & 1.17 \\
90 & 22.22 & 20.94 & 23.16 & 0.84 & 21.47 & 20.02 & 22.76 & 0.92 \\
95 & 23.08 & 21.81 & 24.83 & 1.23 & 20.83 & 20.19 & 22.05 & 0.69 \\
100 & 22.79 & 20.68 & 24.22 & 1.13 & 20.78 & 19.84 & 22.29 & 0.71 \\
\hline
\end{tabular}
\end{center}
\caption{Comparison between results (in terms of IGD) obtained by Cannonical and Memetic Co-evolutionary Algorithms based on 20 independent runs of the tug-of-war MOG variant with $\phi$ = Griewank 2D.}
\label{tab:Griewank2D_results}
\end{table}

\begin{table}[ht!]
\scriptsize
%\vspace{-1em}
\begin{center}
\begin{tabular}{|c||r|r|r|r||r|r|r|r|} \cline{1-9}
    & \multicolumn{4}{c||} {CanonicalCoEvoMOG}
    & \multicolumn{4}{c|}{MemeticCoEvoMOG}
\\\hline
  Generation & Avg	& Min & Max	& Std dev	
  & Avg	& Min & Max	& Std dev	
 \\\hline
5 & 50.06 & 31.19 & 65.34 & 11.21 & 45.12 & 33.66 & 75.94 & 13.48 \\
10 & 64.26 & 34.18 & 117.82 & 25.48 & 38.02 & 27.80 & 55.44 & 9.10 \\
15 & 44.51 & 26.98 & 70.69 & 12.33 & 25.65 & 22.56 & 34.38 & 4.38 \\
20 & 32.81 & 26.98 & 66.00 & 11.97 & 22.74 & 20.08 & 26.94 & 2.03 \\
25 & 28.11 & 21.81 & 52.40 & 8.96 & 21.88 & 20.92 & 25.94 & 1.50 \\
30 & 25.05 & 21.83 & 34.17 & 3.71 & 21.42 & 20.35 & 23.67 & 1.16 \\
35 & 23.85 & 20.58 & 31.20 & 2.89 & 21.62 & 19.97 & 22.98 & 0.92 \\
40 & 23.47 & 20.72 & 33.44 & 3.76 & 21.79 & 19.97 & 23.09 & 0.96 \\
45 & 23.81 & 20.25 & 29.19 & 2.57 & 21.31 & 19.70 & 21.95 & 0.74 \\
50 & 22.87 & 20.80 & 25.21 & 1.48 & 21.97 & 19.46 & 23.69 & 1.11 \\
55 & 22.58 & 21.18 & 25.99 & 1.37 & 21.07 & 19.83 & 21.68 & 0.62 \\
60 & 22.79 & 21.33 & 26.38 & 1.63 & 21.62 & 20.11 & 22.02 & 0.72 \\
65 & 23.85 & 21.12 & 27.83 & 2.28 & 21.28 & 19.97 & 22.70 & 0.82 \\
70 & 22.80 & 20.89 & 27.65 & 1.96 & 20.53 & 19.51 & 22.65 & 1.02 \\
75 & 22.61 & 20.35 & 25.07 & 1.47 & 21.77 & 19.85 & 22.99 & 1.03 \\
80 & 22.06 & 19.95 & 24.24 & 1.17 & 21.18 & 19.84 & 23.27 & 0.98 \\
85 & 22.65 & 20.03 & 24.90 & 1.64 & 21.34 & 20.16 & 22.88 & 0.86 \\
90 & 22.42 & 21.07 & 23.86 & 0.90 & 21.32 & 19.43 & 23.19 & 1.19 \\
95 & 22.36 & 20.57 & 25.39 & 1.47 & 21.39 & 20.46 & 22.68 & 0.67 \\
100 & 22.00 & 20.92 & 24.34 & 1.06 & 21.04 & 19.45 & 22.91 & 1.09 \\
\hline
\end{tabular}
\end{center}
\caption{Comparison between results (in terms of IGD) obtained by Cannonical and Memetic Co-evolutionary Algorithms based on 20 independent runs of the tug-of-war MOG variant with $\phi$ = Griewank 3D.}
\label{tab:Griewank3D_results}
\end{table}

\begin{table}[ht!]
\scriptsize
\begin{center}
\begin{tabular}{|c||r|r|r|r||r|r|r|r|} \cline{1-9}
    & \multicolumn{4}{c||} {CanonicalCoEvoMOG}
    & \multicolumn{4}{c|}{MemeticCoEvoMOG}
\\\hline
  Generation & Avg	& Min & Max	& Std dev	
  & Avg	& Min & Max	& Std dev	
 \\\hline
5 & 226.76 & 99.59 & 383.74 & 95.09 & 126.66 & 37.18 & 322.69 & 85.99 \\
10 & 199.62 & 42.24 & 409.33 & 129.43 & 48.78 & 31.50 & 68.11 & 12.33 \\
15 & 182.56 & 48.71 & 398.80 & 106.92 & 31.29 & 22.81 & 41.76 & 5.78 \\
20 & 159.47 & 39.99 & 427.36 & 116.25 & 23.55 & 20.76 & 27.76 & 2.19 \\
25 & 137.04 & 32.40 & 336.20 & 93.39 & 21.37 & 19.29 & 22.29 & 0.89 \\
30 & 114.79 & 30.98 & 283.01 & 83.79 & 21.12 & 20.38 & 23.07 & 0.81 \\
35 & 98.50 & 27.48 & 203.53 & 63.03 & 21.16 & 20.16 & 22.22 & 0.79 \\
40 & 94.67 & 25.11 & 261.62 & 72.15 & 20.87 & 19.95 & 22.16 & 0.68 \\
45 & 68.56 & 23.42 & 145.00 & 41.81 & 20.83 & 19.03 & 22.55 & 1.29 \\
50 & 55.32 & 22.41 & 122.20 & 32.63 & 20.62 & 19.56 & 22.34 & 0.96 \\
55 & 52.46 & 21.91 & 93.02 & 26.93 & 21.14 & 20.06 & 22.37 & 0.83 \\
60 & 51.33 & 20.29 & 118.80 & 30.57 & 21.10 & 19.50 & 22.20 & 0.95 \\
65 & 43.51 & 21.67 & 76.99 & 20.07 & 21.07 & 19.64 & 22.31 & 0.85 \\
70 & 39.83 & 22.65 & 82.61 & 19.24 & 20.86 & 19.17 & 21.67 & 0.88 \\
75 & 37.38 & 21.74 & 96.71 & 23.09 & 20.92 & 20.09 & 22.59 & 0.75 \\
80 & 32.96 & 21.08 & 71.85 & 15.29 & 20.63 & 18.94 & 21.56 & 0.98 \\
85 & 33.67 & 21.33 & 67.30 & 15.18 & 21.08 & 19.11 & 22.42 & 0.98 \\
90 & 29.98 & 21.45 & 55.10 & 10.64 & 21.17 & 19.63 & 23.59 & 1.36 \\
95 & 29.84 & 21.26 & 64.97 & 13.10 & 20.82 & 18.95 & 22.32 & 1.00 \\
100 & 27.13 & 22.15 & 44.12 & 6.91 & 21.00 & 19.98 & 22.07 & 0.61 \\
\hline
\end{tabular}
\end{center}
\caption{Comparison between results (in terms of IGD) obtained by Cannonical and Memetic Co-evolutionary Algorithms based on 20 independent runs of the tug-of-war MOG variant with $\phi$ = Ackley 2D.}
\label{tab:Ackley2D_results}
\end{table}

\end{document}